\documentclass[usletter]{article} % For LaTeX2e

\usepackage{epsf}
\usepackage{graphicx}
\usepackage{subfigure} 

\usepackage{amsmath}
\usepackage{amssymb}
\usepackage{algorithmic}
\usepackage{algorithm}

\newcommand{\B}{\mathcal B}
\newcommand{\V}{\mathbb V}
\renewcommand{\P}{\mathbb P}
\newcommand{\no}{\epsilon}
\newcommand{\loi}{\mathcal L}

\newlength{\minipagewidth}
\newcommand{\bookbox}[1]{\small
\par\medskip\noindent
\framebox[0.98\columnwidth]{
\begin{minipage}{0.98\minipagewidth} {#1} \end{minipage} } \par\medskip }

\newcommand{\beq}{\begin{equation}}
\newcommand{\eeq}{\end{equation}}

\newcommand{\beqa}{\begin{eqnarray}}
\newcommand{\eeqa}{\end{eqnarray}}

\newcommand{\beqan}{\begin{eqnarray*}}
\newcommand{\eeqan}{\end{eqnarray*}}

\newcommand{\E}{\mathbb{E}}

\newcommand{\alg}{\mathcal A}

\newcommand{\hmu}{\hat{\mu}}

\newcommand{\var}{\sigma^2}

\newcommand{\si}{\sigma}
\newcommand{\hsi}{\hat{\sigma}}

\newcommand{\N}{\mathcal N}

\newcommand{\T}{{\mathcal{T}}}

\newcommand{\ind}[1]{\mathbb I\left\lbrace {#1} \right\rbrace}

\newcommand{\X}{\mathcal X}

\newcommand{\R}{\mathbb R}

\newtheorem{assumption}{Assumption}
\newtheorem{proposition}{Proposition}
\newtheorem{theorem}{Theorem}
\newtheorem{lemma}{Lemma}
\newtheorem{corollary}{Corollary}
\newenvironment{proof}[1][Proof]{\begin{trivlist}
\item[\hskip \labelsep {\bfseries #1}]}{\end{trivlist}}

\usepackage{hyperref}

\usepackage[accepted]{icml2013} 
\icmltitlerunning{Toward Optimal Stratification for Stratified Monte-Carlo Integration}

\begin{document} 

\twocolumn[
\icmltitle{Toward Optimal Stratification for Stratified Monte-Carlo Integration}

% It is OKAY to include author information, even for blind
% submissions: the style file will automatically remove it for you
% unless you've provided the [accepted] option to the icml2013
% package.
\icmlauthor{Alexandra Carpentier}{a.carpentier@statslab.cam.ac.uk}
\icmladdress{Statistical Laboratory, Center for Mathematical Sciences, Wilberforce Road, CB3 0WB Cambridge, United Kingdom}
\icmlauthor{R\'emi Munos}{remi.munos@inria.fr}
\icmladdress{INRIA Lille - Nord Europe, Parc Scientifique de la Haute-Borne, 40 Avenue Halley, 59650 Villeneuve d’Ascq, France}

% You may provide any keywords that you 
% find helpful for describing your paper; these are used to populate 
% the "keywords" metadata in the PDF but will not be shown in the document
\icmlkeywords{Bandit Theory, Monte-Carlo Integration}

\vskip 0.3in
]

\begin{abstract}
We consider the problem of adaptive stratified sampling for Monte Carlo integration of a noisy function, given a finite budget $n$ of noisy evaluations to the function. We tackle in this paper the problem of adapting to the function at the same time the number of samples into each stratum and the partition itself. More precisely, it is interesting to refine the partition of the domain in area where the noise to the function, or where the variations of the function, are very heterogeneous. On the other hand, having a (too) refined stratification is not optimal. Indeed, the more refined the stratification, the more difficult it is to adjust the allocation of the samples to the stratification, i.e.~sample more points where the noise or variations of the function are larger.
We provide in this paper an algorithm that selects online, among a large class of partitions, the partition that provides the optimal trade-off, and allocates the samples almost optimally on this partition.
\end{abstract}

\vspace{-0.7cm}

\section{Introduction}
\vspace{-0.2cm}

The objective of this paper is to provide an efficient strategy for integrating a noisy function $F$. The learner can sample $n$ times the function. If it samples the function at a time $t$ in a point $x_t$ of the domain $\X$ \textit{that it can choose to its convenience}, it obtains the noisy sample $F(x_t,\no_t)$, where $\no_t$ is drawn independently at random from some distribution $\loi_{x_t}$, where $\loi_x$ is a probability distribution that depends on $x$.

If the variations of the function $F$ are known to the learner, an efficient strategy is to sample more points in parts of the domain $\X$ where the variations of $F$ are larger. This intuition is explained more formally in the setting of \textit{Stratified Sampling} (see e.g.~\cite{rubinstein2008simulation}).%, i.e.~a setting where one divides the domain in $K$ strata and allocates samples in the strata.
\\
 More precisely, assume that the domain $\X$ is divided in $K_{\N}$ regions (according to the usual terminology of stratified sampling, we refer to these regions as strata) that form a partition $\N$ of $\X$. It is optimal (for an oracle) to allocate a number of points in each stratum proportional to the measure of the stratum times a quantity depending of the variations of $F$ in the stratum (see Subsection 5.5 of \cite{rubinstein2008simulation}). We refer to this strategy as optimal oracle strategy for partition $\N$. %and write $\frac{\Sigma_{\N}^2}{n}$ the mean squared error (with respect to the integral of $F$) of the estimate outputted by the optimal oracle strategy.

The problem is that the variations of the function $F$ in each stratum of $\N$ are unknown to the learner. 
 In the papers~\cite{A-EtoJou10, Grover, MC-UCB}, the authors expose the problem of, at the same time, estimating the variations of $F$ in each stratum, and allocating the samples optimally among the strata according to these estimates.
\\
%More precisely, in~\cite{rapp-tech-MC-UCB}\footnote{This is the detailed version of~\cite{MC-UCB}, where the bounds are enhanced.}, the authors provide an asymptotically consistent algorithm whose pseudo-risk\footnote{We define precisely later in the paper the notion of pseudo-risk. It is a proxy for the mean squared error of the estimate of the integral.} is bounded by $\frac{\Sigma_{\N}^2}{n} + C_{\min} \Sigma_{\N} \frac{K_{\N}^{1/3}}{n^{4/3}}$, where $C_{\min}$ is a constant. 
Up to some variation in efficiency or assumptions, these papers provide learners that are indeed able to learn about the variations of the function and allocate optimally the samples in the strata, up to a negligible term. However, all these papers make explicit in the theoretical bounds, or at least intuitively, the existence of a natural trade-off in terms of the refinement of the partition. The more refined the partition (especially if it gets more refined where variations of $F$ are larger), the smaller the variance of the estimate outputted by the optimal oracle strategy. However, the larger the error of an adaptive strategy with respect to this optimal oracle strategy, since the more strata there are, the harder it is to adapt to each stratum.
%If the domain is wisely stratified, according to $F$, and in many strata, then $\frac{\Sigma_{\N}^2}{n}$ is small. However the term $\frac{K_{\N}^{1/3}}{n^{4/3}}$ in the bound depends also of the partition of the space and increases with the number of strata. The intuition behind this fact is that the learner has to learn the variations of the function inside each stratum, and the more strata there are, the harder it is.

%It is however not a good strategy to divide the domain in only one stratum. Indeed, the interest of stratified sampling is that the variance is reduced by sampling more deterministically, and if the problem of learning the variances inside the strata did not exist, a good strategy would be to have as many strata as possible, that is to say $K=n$ strata. It indeed makes the (rescaled) variance of the static oracle, $\Sigma_{\N}^2$, diminish.
%\\
It is thus important to adapt also the partition to the function, and refine more the strata where variations of the function $F$ are larger, while at the same time limiting the number of strata.
As a matter of fact, a good partition of the domain is such that, inside each stratum, the values taken by $F$ are as homogeneous as possible (see Subsection 5.5 of \cite{rubinstein2008simulation}), while at the same time the number of strata is not too large.
\\
There are some recent papers on how to stratify efficiently the space, e.g.~\cite{glasserman1999asymptotically,kawai2010asymptotically, A-EtoForJouMou11,carpentier2012minimax, carpentier2012online}. More specifically, in the recent paper~\cite{A-EtoForJouMou11}, the authors propose an algorithm for performing this task online and efficiently. They do not provide proofs of convergence for their algorithm, but they give some properties of optimal stratified estimate when the number of strata goes to infinity, notably convergence results under the optimal allocation. They also give some intuitions on how to split efficiently the strata. Having an asymptotic vision of this problem prevents them however from giving clear directions on how exactly to adapt the strata, as well as from providing theoretical guarantees. In paper~\cite{carpentier2012minimax}, the authors propose to stratify the domain according to some preliminary knowledge on the class of smoothness of the function. They however fix the partition before sampling and thus do not consider \textit{online} adaptation of the partition to the function. Finally, although considering online adaptation of the partition to the function, the paper~\cite{carpentier2012online} considers the specific and somehow very different\footnote{In this setting where the function $F$ is noiseless and very regular, efficient strategies share ideas with quasi Monte-Carlo strategies, and the number of strata should be almost equal to the budget $n$.} setting where the noise $\epsilon$ to the function $F$ is null, and where $F$ is differentiable according to $x$.

\vspace{-0.3cm}
\paragraph{Contributions:}

We consider in this paper the problem of designing efficiently and according to the function a partition of the space, and of allocating the samples efficiently on this partition. More precisely, our aim is to build an algorithm that allocates the samples almost in an oracle way on the best possible partition (adaptive to the function $F$, i.e.~that solves the trade-off that we named before) in a large class of partitions. We consider in this paper the class of partition to be the set of partitions defined by a hierarchical partitioning of the domain (as for instance what was considered in~\cite{bubeck2008online} for function optimization). %A reasonable strategy is to refine the partition in areas of the domain \textit{if and only if} the performance of MC-UCB on the more refined partition is better than the performance of MC-UCB on the initial partition. Our contributions are the following:
\vspace{-0.3cm}
\begin{itemize}
\item We provide new, to the best of our knowledge, ideas for sampling a domain very homogeneously, i.e.~such that the samples are well scattered. The sampling schemes we introduce share ideas with low discrepancy schemes (see e.g.~\cite{niederreiter2010quasi}), and provide some theoretic guarantees for their efficiency. 
\item We provide an algorithm, called Monte-Carlo Upper Lower Confidence band. We prove that it manages to at the same time select an optimal partition of the hierarchical partitioning and then to allocate the samples in this partition almost as an oracle would do. More precisely, we prove that its pseudo-risk is smaller, up to a constant, than the pseudo-risk of MC-UCB on \textit{any} partition of the hierarchical partitioning.
%Two distinct analyses corresponding to two different configurations of this problem: i) If the stratas are homogeneous or if the budget is big (i.e. if the problem is ``simple''), we are able to state that we deviate from the targetted variance from only $\widetilde O(n^{-3/2})$ while ii) if the stratas are inhomogeneous or the budget is small (i.e. the problem is difficult), then we deviate from $O(n^{-4/3})$.
\end{itemize}
\vspace{-0.3cm}
The rest of the paper is organised as follows. In Section~\ref{s:preliminaries} we formalise the problem and introduce the notations used throughout the paper. We also remind the problem independent bound for algorithm MC-UCB. Section~\ref{s:algo.mculcb} presents algorithm MC-ULCB, and its bound on the pseudo-risk. After a technical part on notations, we introduce what we call Balanced Sampling Scheme (BSS) and a variant of it, BSS-A. These are sampling schemes for allocating samples in a random yet almost low discrepancy way, on a domain. Algorithm MC-ULCB that we present afterwards relies heavily on them. We also discuss the results, and finally conclude the paper.

\vspace{-0.2cm}

\section{Preliminaries}\label{s:preliminaries}
\vspace{-0.2cm}

\subsection{The function}
\vspace{-0.2cm}

Consider a noisy function  $F:(x,\no) \in \X \times \Omega \rightarrow \R$.\\
In this definition, $\X$ is the domain on which the learner can \textit{choose} in which point $x$ to sample, and $\Omega$ is a space on which the noise to the function $\no$ is defined. We define for any $x \in \X$ the distribution of noise $\no$ conditional to $x$ as $\loi_{x}$. We also define a \textit{finite} measure $\nu$ on $\X$ corresponding to a $\sigma-$algebra whose sets belong to $\X$. Without loss of generality, we assume that $\nu(\X) = 1$ ($\nu$ is a probability measure).

The objective of the learner is to sample the domain $\X$ in order to build an efficient estimate of the integral of the noisy function $F$ according to the measure $(\nu,\loi_{x}|x)$, that is to say $\int_{\X} \E_{\no_x \sim \loi_x} F(x,\no_x) d(\nu)(x)$. The learner can sample sequentially the function $n$ times, and observe noisy samples. When sampling the function at time $t$ in $x_t$, it observes a noisy sample $F(x_t,\no_t)$. The noise $\no_t \sim \loi_{x_t}$ conditional to $x_t$ is independent of the previous samples $(x_i,\no_i)_{i<t}$.

For any point $x \in \X$, define
\vspace{-0.2cm}
\begin{small}
\begin{align*}
g(x) = \E_{\no \sim \loi_x} F(x, \no) \hspace{0.2cm} and  \hspace{0.2cm} s(x) = \sqrt{\E_{\no \sim \loi_x}\Big[\big(F(x, \no) - g(x) \big)^2 \Big]}. 
\end{align*}
\end{small}
\vspace{-0.2cm}
We state the following Assumption on the function
\vspace{-0.1cm}
\begin{assumption}\label{ass:boundness}
We assume that both $g$ and $s$ are bounded in absolute value by a constant $f_{\max}$.
Let $\upsilon(x,\no) = \frac{F(x, \no) - g(x)}{s(x)}$ (if $s(x)=0$, set $\upsilon(x,\no) = 0$). We assume that $\exists b$ such that $\forall \lambda <\frac{1}{b}$,
\begin{center}
$\E_{\no \sim \loi_x}\Big[ \exp(\lambda \upsilon(x,\no) ) \Big] \leq \exp\Big( \frac{\lambda^2 }{2(1 - \lambda b)}\Big)$,     and       $\E_{\no \sim \loi_x}\Big[ \exp(\lambda \upsilon(x,\no)^2 - \lambda) \Big] \leq \exp\Big( \frac{\lambda^2 }{2(1 - \lambda b)}\Big).$
\end{center}
\end{assumption}
\vspace{-0.2cm}
Assumption~\ref{ass:boundness} means that the variations coming from the noise in $F$, although potentially unbounded, are not too large\footnote{This assumption implies that the variations induced by the noise are sub-Gaussian. It is actually slightly stronger than the usual sub-Gaussian assumption. Nevertheless, e.g.~bounded random variables and Gaussian random variables satisfy it.}. We believe that it is rather general. In particular, it is satisfied if $F$ is bounded, or also for e.g.~a bounded function perturbed by an additive, heterocedastic, (sub-)Gaussian noise.

\vspace{-0.2cm}
\subsection{Notations for a hierarchical partitioning}
\vspace{-0.2cm}

The strategies that we are going to consider for integration are allowed to choose where to sample the domain. In order to do that, the strategies we consider will partition the domain $\X$ into strata and sample randomly in the strata. In theory the stratification is at the discretion of the strategy and can be arbitrary. However in practice, we will consider strategies that rely on given hierarchical partitioning.

Define a dyadic hierarchical partitioning $\mathbf{\T}$ of the domain $\X$. More precisely, we consider a set of partitions of $\X$ at every depth $h \geq 0$: for any integer $h$, $\X$ is partitioned into a set of $2^h$ strata $\X_{[h,i]}$, where $0 \leq i \leq 2^{h} - 1$. This partitioning can be represented by a dyadic tree structure, where each stratum $\X_{[h,i]}$ corresponds to a node $[h, i]$ of the tree (indexed by its depth $h$ and index $i$). Each node $[h, i]$ has $2$ children nodes $[h+1, 2i]$ and $[h+1, 2i+1]$. In addition, the strata of the children form a sub-partition of the parents stratum $\X_{[h,i]}$ . The root of the tree corresponds to the whole domain $\X$.

We state the following assumption on the measurability and on the measure of any stratum of the hierarchical partitioning.
\vspace{-0.2cm}
\begin{assumption}\label{ass:stratummeasurables}
 $\forall [h,i] \in \mathbf{\T}$, the stratum $\X_{[h,i]}$ is measurable according to the $\sigma-$algebra on which the probability measure $\nu$ is defined.
\end{assumption}
\vspace{-0.2cm}
We write $w_{[h,i]}$ the measure of stratum $\X_{[h,i]}$, i.e.~$w_{[h,i]} = \nu(\X_{[h,i]})$. We also assume that the hierarchical partitioning is such that all the strata of a given depth have same measure, i.e.~$w_{[h,i]} = w_h$.
\vspace{-0.2cm}
\begin{assumption}\label{ass:stratumequal}
 $\forall [h,i] \in \mathbf{\T}$, the children strata of $[h,i]$ are such that $w_{h+1} = \nu(\X_{[h+1,2i]}) = \nu(\X_{[h+1,2i+1]}) = \frac{\nu(\X_{[h,i]})}{2} = \frac{w_{h}}{2}$.
\end{assumption}
\vspace{-0.2cm}
If for example $\X=[0,1]$, a hierarchical partitioning that satisfies the previous assumptions with the Lebesgue measure is illustrated in Figure~\ref{fig:diffvalC}.
\begin{center}
\begin{figure}[htbp]
\begin{center}
\includegraphics[width=8cm]{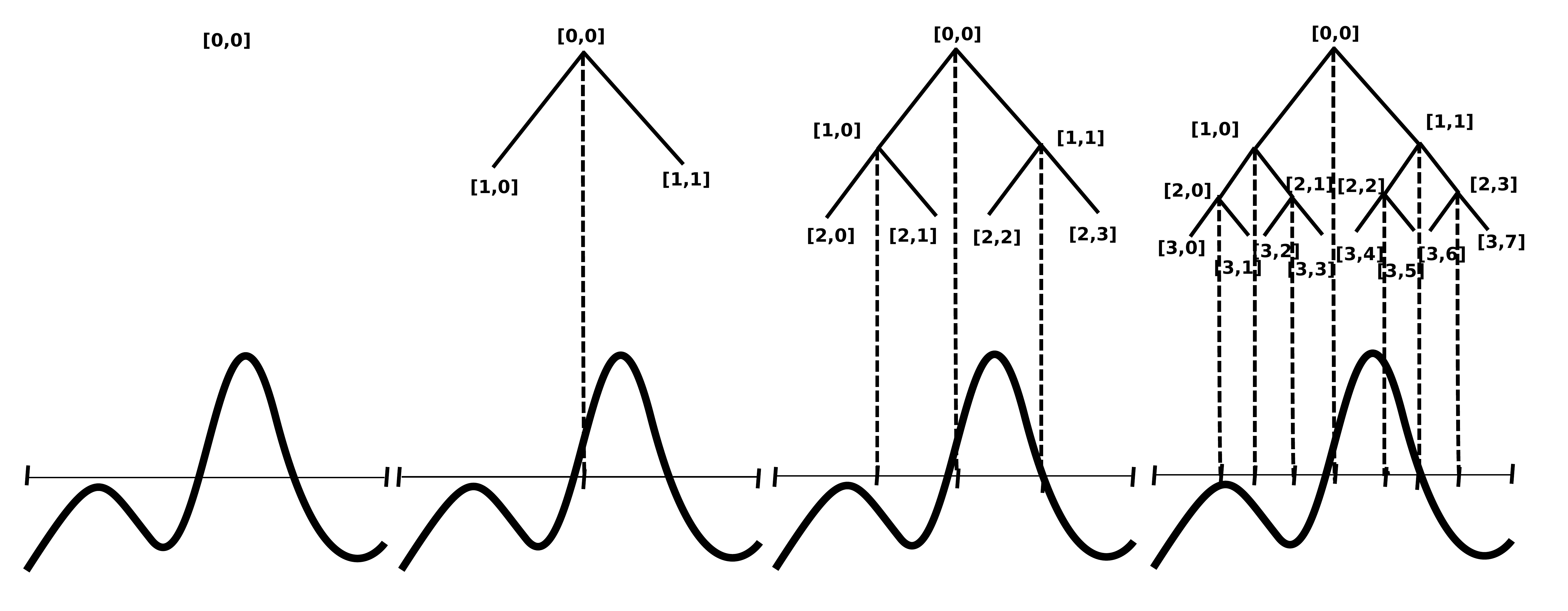}
\vspace{-0.6cm}
\caption{Example of hierarchical partitioning in dimension $1$.} \label{fig:diffvalC}
\vspace{-0.4cm}
\end{center}
\end{figure}
\end{center}
%We consider two sets of nodes $\N_1$ and $\N_2$. We write $\N_1 \bigcup \N_2$ the set of nodes that are in $\N_1$ or $\N_2$. We write $\N_1 \setminus \N_2$ the set of nodes that contains all nodes that are in $\N_1$ but not in $\N_2$ (i.e.~$\N_1 \bigcap \N_2^c$).
We write mean and variance of stratum $\X_{[h,i]}$ the mean and variance of a sample of the function $F$, collected in the point $X$, where $X$ is drawn at random according to $\nu$ conditioned to stratum $\X_{[h,i]}$. %Assume that a sample is collected at time $t$ in stratum $k$. Let $Y_t=f(X,t)$ be this sample, where $X$ is a random variable distributed as $\nu$ restricted on stratum $\X_{[h,i]}$, that is to say $\nu_{\X_{[h,i]}}$.
We write
\vspace{-0.2cm}
\begin{small}
\begin{align*}
\mu_{[h,i]} = \E_{X \sim \nu_{\X_{[h,i]}}}\Big[ \E_{\no \sim \loi_X}[ F(X,\no)] \Big] =\frac{1}{w_h}\int_{\X_{[h,i]}} g(x) d\nu(x),
\end{align*}
\end{small}
\noindent
\vspace{-0.1cm}
the mean and
\vspace{-0.2cm}
\begin{footnotesize}
\begin{align*}
\si^2_{[h,i]} &=  \V_{X \sim \nu_{\X_{[h,i]}} ,\no \sim \loi_X} [F(X,\no)] \\
&= \frac{1}{w_h}\int_{\X_{[h,i]}} \Big(g(x) - \mu_{[h,i]} \Big)^2 d\nu(x) + \frac{1}{w_h}\int_{\X_{[h,i]}} s^2(x) d\nu(x),
\end{align*}
\end{footnotesize}
\vspace{-0.2cm}
\noindent
the variance (we remind that $g$ and $s$ are defined in Assumption~\ref{ass:boundness}).

\vspace{-0.2cm}
%Building a hierarchical partitioning of the space is often

% 
% Let $\N$ be an ensemble of nodes that represents a partition of the space. We assume that we pick uniformly and independently, in each node $[h,i]$ of $\N$ a number $T_{[h,i]}$ of random samples. We denote by $\Big(X_{[h,i],t}\Big)_{[h,i] \in \N, t \leq T_{[h,i]}}$ these samples, and we call $\hmu_{[h,i]} = \frac{1}{T_{[h,i]}} \sum_{t=1}^{T_{[h,i]}}X_{[h,i],t}$. We estimate the integral of $f$ on $[0,1]$ by $\hmu = \sum_{[h,i] \in \N} w_h \hmu_{[h,i]}$.
% 
% 
% Note that we have
% 
% 
% \begin{equation*}
%  \E(\hmu) = \sum_{[h,i] \in \N} w_h \mu_{[h,i]} = \sum_{[h,i] \in \N} \int_{U_{[h,i]}} f(u) du = \int_0^1 f(u) du = \mu.
% \end{equation*}
% 
% Note also that 
% 
% \begin{equation*}
%  \V(\hmu) = \sum_{[h,i] \in \N} \frac{w_h^2 \V(\hmu_{[h,i]})}{T_{[h,i]}} = \sum_{[h,i] \in \N} \frac{w_h^2 \var_{[h,i]}}{T_{[h,i]}}.
% \end{equation*}

\subsection{Pseudo-performance of an algorithm and optimal static strategies}
\vspace{-0.2cm}

We denote by $\alg$ an algorithm that allocates the budget $n$ and returns a partition $\N_n = \Big(\X_{[h,i]} \Big)_{[h,i] \in \N_n}$ included in the hierarchical partitioning $\T$ of the domain. In each node $[h,i]$ of $\N_n$, algorithm $\alg$ allocates uniformly $T_{[h,i],n}$ random samples. We write $\Big(X_{[h,i],t}\Big)_{[h,i] \in \N_n, t \leq T_{[h,i],n}}$ these samples, and we write $\hmu_{[h,i],n} = \frac{1}{T_{[h,i],n}} \sum_{t=1}^{T_{[h,i],n}}X_{[h,i],t}$ the empirical mean built with these samples. We estimate the integral of $F$ on $\X$ by $\hmu_n = \sum_{[h,i] \in \N_n} w_h \hmu_{[h,i],n}$. This is the estimate returned by the algorithm.

If $\N_n$ is fixed as well as the number $T_{[h,i],n}$ of samples in each stratum, and if the $T_{[h,i],n}$ samples are independent and chosen uniformly according to the measure $\nu$ restricted to each stratum $\X_{[h,i]}$, we have
\begin{small}
\begin{equation*}
\vspace{-0.2cm}
 \E(\hmu_n) = \sum_{[h,i] \in \N_n} w_h \mu_{[h,i]} = \sum_{[h,i] \in \N_n} \int_{\X_{[h,i]}} g(u) d\nu(u) = \mu,
\vspace{-0.2cm}
\end{equation*}
\end{small}
and also
\begin{equation*}
\vspace{-0.2cm}
 \V(\hmu_n) = \sum_{[h,i] \in \N_n} w_h^2 \E(\hmu_{[h,i],n}-\mu_{[h,i]})^2 = \sum_{[h,i] \in \N_n} \frac{w_h^2 \var_{[h,i]}}{T_{[h,i],n}},
\vspace{-0.2cm}
\end{equation*}
where the expectations and variance are computed with respect to the samples collected in the strata.

For a given algorithm $\alg$, we denote by \emph{pseudo-risk} the quantity
\begin{equation}\label{risk}
\vspace{-0.2cm}
 L_n(\alg) = \sum_{[h,i] \in \N_n} \frac{w_h^2 \var_{[h,i]}}{T_{[h,i],n}}.
\vspace{-0.2cm}
\end{equation}
This measure of performance is discussed more in depths in papers~\cite{Grover, rapp-tech-MC-UCB}. In particular, paper~\citep{rapp-tech-MC-UCB} links it with the mean squared error.
%where $\N_n$ is such that the samples in all nodes $[h,i]$ of $\N_n$ are allocated uniformly at random according to the Lebesgue measure restricted to $\X_{[h,i]}$.

Note that if, for a given partition $\N$, an algorithm $\alg_{\N}^*$ would have access the variances $\var_{[h,i]}$ of the strata in $\N$, it could allocate the budget in order to minimise the pseudo-risk, by choosing to pick in each stratum $\X_{[h,i]}$ (up to rounding issues) $T_{[h,i]}^* = \frac{w_h \si_{[h,i]}n}{\sum_{x \in \N} w_x \si_{x}}$ samples. The pseudo risk for this oracle strategy is then
\vspace{-0.2cm}
\begin{equation}\label{BestStrat}
 L_n(\alg_{\N}^*) = \frac{\Big(\sum_{[h,i] \in \N} w_h \si_{[h,i]}\Big)^2}{n} = \frac{\Sigma_{\N}^2}{n},
\vspace{-0.2cm}
\end{equation}
where we write $\Sigma_{\N} = \sum_{x\in \N} w_x \si_{x}$. We also refer, in the sequel, as optimal allocation (for a partition $\N$), to $\lambda_{[h,i],\N} = \frac{w_h \si_{[h,i]}}{\Sigma_{\N_n}}$. Even when the optimal allocation is not realizable because of rounding issues, it can still be used as a benchmark since the quantity $L_n(\alg_{\N}^*)$ is a lower bound on the variance of the estimate outputted by any oracle strategy.

%We define the pseudo-risk on partition $\N$ in the case when the samples within each stratum $\X_{[h,i]}$ are chosen \textit{uniformly at random} in the stratum according to the measure $\nu_{\X_{[h,i]}}$. In this article, we however do not sample uniformly at random in each stratum of partition $\N$, but according to a sampling scheme, called USS, that we introduce in the following Section. We prove that the variance of the empirical mean of the samples collected with this sampling scheme is smaller than the variance when sampling uniformly at random in stratum $\X_{[h,i]}$, which justifies the use of this scheme.
%It is however consistent to sample in an other way (not uniformly at random) inside the strata, and then prove that the variance of the estimate outputted with the alternative way of sampling has a risk with respect to the integral of the function on the stratum that is smaller than the variance of the estimate outputted by sampling uniformly in the stratum.
%Note now, as a very important remark on this pseudo-risk, that it only makes sense to consider a partition on which (i) either the samples in each stratum are pulled uniformly or (ii) the sample are not allocated uniformly, but it is possible to show that the used allocation is more efficient in terms of variance than the uniform one.

\subsection{Main result for algorithm MC-UCB and point of comparison}
\vspace{-0.2cm}

Let us consider a fixed partition $\N$ of the domain, and write $K_{\N}$ for the number of strata it contains. We first recall (and slightly adapt) one of the main results of paper~\cite{rapp-tech-MC-UCB} (Theorem 2). It provides a result on the pseudo-risk of an algorithm called MC-UCB. This algorithm takes some parameters linked to upper bounds on the variability of the function\footnote{It is needed that the function is bounded and that the noise to the function is sub-Gaussian.}, a small probability $\delta$, and the partition $\N$.  MC-UCB builds, for each stratum in the \textit{fixed}\footnote{It is very important to note that the partition is fixed for this algorithm and that it only adapts the allocation to the function.} partition $\N$, an upper confidence band (UCB) on it's standard deviation, and allocates the samples proportionnal to the measure of each stratum times this UCB. Its pseudo-risk is bounded in high probability by $\frac{\Sigma_{\N}^2}{n} + \Sigma_{\N}  O(\frac{K_{\N}^{1/3}}{n^{4/3}})$.
This theorem holds also in our setting. The fact that the measure $\nu$ is finite together with Assumptions~\ref{ass:stratummeasurables} and~\ref{ass:boundness} imply that the distribution of the samples obtained by sampling in the strata are sub-Gaussian (as a bounded mixture of sub-Gaussian random variables). We remind and slightly improve this theorem.
\vspace{-0.2cm}
\begin{theorem}\label{prop:m-regret}
%Let $K$ be the number of strata in $\N$.
Under Assumptions~\ref{ass:stratummeasurables} and~\ref{ass:boundness}, the pseudo-risk of MC-UCB\footnote{In order to fit with the assumptions of this paper, we redefine $\forall x \in \N$ and $\forall t \leq n$ the upper confidence bound defined in the original paper as $B_{x,t} = \frac{1}{T_{x,t-1}} w_x\Big(\hsi_{x,t} +  \frac{A}{\sqrt{T_{x,t}}}\Big)$.} launched on partition $\N$ with parameters $f_{\max}$, $b$ and $\delta$ is bounded, if $n \geq 4K$, with probability $1-\delta$,
\begin{equation*}
\vspace{-0.2cm}
L_{n}(\alg_{MC-UCB})   \leq \frac{\Sigma_{\N}^2}{n} + C_{\min} \Sigma_{\N} \sum_{x \in \N} \frac{w_x^{2/3}}{n^{4/3}},
\vspace{-0.2cm}
\end{equation*}
where $C_{\min} = (4\sqrt{2}\sqrt{A} + 3 f_{\max} A)$ and $A = 2\sqrt{2(1 + 3b + 4f_{\max})\log(4n^2(3f_{\max})^3/\delta)}$.
\end{theorem}
\vspace{-0.2cm}
The bound in this Theorem is slightly sharper than in the original paper. The (improved) proof is in the Supplementary Material, see Appendix~\ref{proof:algo}

We will use in the sequel the bound in this Theorem as a benchmark for the efficiency of any algorithm that adapts the partition. %More precisely, we compare the pseudo-risk of the strategies that we propose with the pseudo-risk of algorithm MC-UCB launched on partitions of the hierarchical partitioning of the domain.
The aim will be to construct a strategy whose pseudo-regret is almost as small as the minimum of this bound over a large class of partitions (e.g.~the partitions defined by the hierarchical partitioning). In paper~\citep{carpentier2012minimax}, it was proved that this bound is minimax optimal which makes it a sensible benchmark.%prove that our strategies are almost as efficient as MC-UCB launched on the partition (of the hierarchical partitioning of the domain) that minimises this bound.

The bound in this Theorem depends on two terms. The first, $\frac{\Sigma_{\N}^2}{n}$, which is the oracle optimal variance of the estimate on partition $\N$, decreases with the number of strata, and more specifically if the strata are ``well-shaped'' (i.e.~more strata where the variations of $g$ and $s$ are larger). On the other hand, the second term, $ \sum_{x \in \N} \frac{w_x^{2/3}}{n^{4/3}}$, increases when the partition is more refined. There are however two extremal situations for this term, leading to two very different behaviours with the number of strata. If the strata have all the same measure $\frac{1}{K_{\N}}$ where $K_{\N}$ is the number of strata in partition $\N$, then $\sum_{x \in \N} \frac{w_x^{2/3}}{n^{4/3}} = \frac{K_{\N}^{1/3}}{n^{4/3}}$. Now if the partition is very localised (i.e.~exponential decrease of the measure of the strata), then \textit{whatever} the number of strata, $\sum_{x \in \N} \frac{w_x^{2/3}}{n^{4/3}}$ is of order $O(\frac{1}{n^{4/3}})$, and the number of strata $K_{\N}$ has no more influence than a constant.\\
These two facts enlighten the importance of adapting the \textit{shape} of the partition to the function by having potentially strata of heterogeneous measure.%This bound is thus more refined than the original one, and is thus more suitable to really adapt to the trade-off in terms of shape and number of strata, for building the optimal partition of the domain.

%The bound on the pseudo-risk of MC-UCB launched on a partition $\N$ is $\frac{\Sigma_{\N}^2}{n} + C_{\min} \Sigma_{\N_n} \sum_{x \in \N} \frac{w_x^{2/3}}{n^{4/3}}$. A partition

% for the adaptation of a partition to the function and the available budget: the objective is to design an algorithm whose pseudo-risk is smaller than this bound on the pseudo-risk of MC-UCB on any partitions of the hierarchical partitioning.

\vspace{-0.2cm}
\section{Algorithm MC-ULCB}\label{s:algo.mculcb}
\vspace{-0.2cm}

\subsection{Additional definitions for algorithm MC-ULCB}
\vspace{-0.2cm}

Let $\delta>0$. We first define $A = 2\sqrt{2(1 + 3b + 4f_{\max})\log(4n^2(3f_{\max})^3/\delta)}$ where $f_{\max}$ and $b$ are chosen such that they satisfy Assumption~\ref{ass:boundness}. Set also for any $h$, $t_{h} = \lfloor A w_{h}^{2/3} n^{2/3}  \rfloor$.

Let $[h,i]$ be a node of the hierarchical partitioning.

Assume that the children $([h+1,2i], [h+1,2i+1])$ of node $[h,i]$ have received at least $t_{h+1}$ samples (and stratum $\X_{[h,i]}$ has received at least $2t_{h+1}$ samples). The standard deviations $\hsi_{[h+1,j]}$ (for $j \in \{2i,2i+1\}$) are computed using the first $t_{h+1}$ samples only:
\begin{equation}\label{eq:estim-var2}
\vspace{-0.2cm}
\hsi_{[h+1,j]} = \sqrt{\frac{1}{t_{h+1}} \sum_{u=1}^{t_{h+1}} (X_{[h+1,j],u}- \frac{1}{t_{h+1}} \sum_{k=1}^{t_{h+1}} X_{[h+1,j],k})^2},
\vspace{-0.2cm}
\end{equation}
where $X_{[h+1,j],u}$ is the $u$-th sample in stratum $\X_{[h+1,j]}$. %We remind that the samples $\{X_{[h+1,2i],u} \}_u$ and $\{X_{[h+1,2i+1],u} \}_u$ are subsets of the samples $\{X_{[h,2],u} \}_u$ by definition of the stratification and of the sample schemes.\\
We also introduce another estimate for the standard deviation $\hsi_{[h,i]}$, namely $\tilde \si_{[h,i]}$, which is computed with the first $2 t_{h+1}$ samples in stratum $\X_{[h,i]}$ (and not with the first $t_h$ samples as $\hsi_{[h,i]}$):
\begin{equation}\label{eq:estim-var22}
\vspace{-0.2cm}
\tilde \si_{[h,i]} = \sqrt{\frac{1}{2t_{h+1}} \sum_{u=1}^{2t_{h+1}} (X_{[h,i],u}- \frac{1}{2t_{h+1}} \sum_{k=1}^{2t_{h+1}} X_{[h,i],k})^2}.
\vspace{-0.2cm}
\end{equation}
We use this estimate for technical purposes only.

We now define by induction the value $r$ for any stratum $\X_{[h,i]}$. We initialise the $r$ when there is enough points i.e.~at least $t_{0}$ points in stratum $\X_{[0,0]}$, by $r_{[0,0]} = \hsi_{[0,0]} - \frac{c\sqrt{A}}{n^{1/3}}$. Assume that $r_{[h,i]}$ is defined. Whenever there are at least $t_{[h+1]}$ points in strata $\X_{[h+1,j]}$ for $j \in \{2i,2i+1 \}$, we define the value $r_{[h+1,j]}$ for $j \in \{2i,2i+1\}$ (and $j^-$ the other) as
\begin{small}
\begin{align}
\vspace{-0.2cm}
&r_{[h+1,j]} = \Big( \frac{w_{h+1} \hsi_{[h+1,j]} + c\sqrt{A}\frac{w_{h+1}^{2/3}}{n^{1/3}}}{w_{h} \tilde \si_{[h,i]}} \Big) r_{[h,i]}\label{eq:r}\\
\vspace{-0.2cm}
&\times \ind{w_{h+1} \hsi_{[h+1,j^-]} - w_{h+1} \hsi_{[h+1,j]} \geq 2c \sqrt{A}\frac{w_{h+1}^{2/3}}{n^{1/3}}} \nonumber\\
\vspace{-0.2cm}
&+ \Big( \frac{w_{h+1} \hsi_{[h+1,j]} - c\sqrt{A}\frac{w_{h+1}^{2/3}}{n^{1/3}}}{w_{h} \tilde  \si_{[h,i]}} \Big) r_{[h,i]}\nonumber\\
\vspace{-0.2cm}
&\times \ind{w_{h+1} \hsi_{[h+1,j^-]} - w_{h+1} \hsi_{[h+1,j]} \leq - 2c\sqrt{A}\frac{w_{h+1}^{2/3}}{n^{1/3}}} \nonumber\\
\vspace{-0.2cm}
&+ \min\Big( \frac{w_{h+1}  \min\big( \hsi_{[h+1,j]},\hsi_{[h+1,j^-]}\big) + c\sqrt{A}\frac{w_{h+1}^{2/3}}{n^{1/3}}}{w_{h} \tilde \si_{[h,i]}} , \frac{1}{2} \Big)r_{[h,i]} \nonumber\\
\vspace{-0.2cm}
&\times \ind{|w_{h+1} \hsi_{[h+1,j^-]} - w_{h+1} \hsi_{[h+1,j]}| \leq  2c\sqrt{A}\frac{w_{h+1}^{2/3}}{n^{1/3}}}, \nonumber
\vspace{-0.2cm}
\end{align}
\end{small}
where  $c = (8\tilde \Sigma +1)\sqrt{A}$, $\tilde \Sigma = \hsi_{[0,0]} + \frac{\sqrt{A}}{n^{1/3}}$. It is either a (proportional) upper, or a (proportional) lower confidence bound on $w_{[h+1,j]}\si_{[h+1,j]}$. It is a (proportional) upper confidence bound for the stratum $[h+1,j]$ that has the smallest empirical standard deviation, and a (proportional) lower confidence bound for the other. If the quantities $w_{[h+1,j]}\hsi_{[h+1,2i]}$ and $w_{[h+1,j]}\hsi_{[h+1,2i+1]}$ are too close, we set the same value to both sub-strata. The quantities $r_{[h,i]}$ are key elements in algorithm MC-ULCB, and they account for the name of the algorithm (Monte Carlo Upper Lower Confidence Bound).

Additional to that, we define the technical quantities $H = \lfloor \frac{\log\big((3 f_{\max})^3 n\big)}{\log(2)} \rfloor +1$, $B=38\sqrt{2A} c  (1 + \frac{1}{\tilde\Sigma})$ and $C_{\max}'= \max(B,14 H c\sqrt{A}) + 2\sqrt{A}$.

\vspace{-0.2cm}
\subsection{Sampling Schemes}
\vspace{-0.2cm}

The algorithm MC-ULCB that we will consider in the next Subsection works by updating a partition of the domain, refining it more where it seems necessary (i.e.~where the algorithms detects that $g$ or $s$ have large variations). In order to do that, the algorithm needs to split some nodes $[h,i]$ in their children nodes. We thus need guarantees on the number of samples in each child node $[h+1,2i]$ and $[h+1,2i+1]$, when there are $t$ samples in $[h,i]$. More precisely, we would like to have, up to rounding issues, $t/2$ samples in each child node.\\
The problem is that usual sampling procedures do not guarantee that. In particular, if one chooses the naive idea for sampling stratum $\X_{[h,i]}$, i.e.~collect $t$ samples independently at random according to $\nu_{\X_{[h,i]}}$, then there is no guarantees on the exact numbers of samples in $[h+1,2i]$ and $[h+1,2i+1]$. However, we would like that the sampling scheme that we use conserve the nice properties of sampling according to $\nu_{\X_{[h,i]}}$, i.e.~that the empirical mean built on the samples remains an unbiased estimate of $\mu_{[h,i]}$ and that it has a variance smaller than or equal to $\si_{[h,i]}^2/t$.\\
This is one of the reasons why we need alternative sampling schemes

\vspace{-0.2cm}
\paragraph{The Balanced Sampling Scheme}
\vspace{-0.1cm}

We first describe what we call Balanced Sampling Scheme (BSS). %We will use it for the main algorithm that we describe in this paper.

We design this sampling scheme in order to be able to divide at any time each stratum, so that at any time, the number of points in each sub-stratum is proportional to the measure of the sub-stratum (up to one sample of difference).% This means that we need to sample uniformly on the domain, almost in a low-discrepancy way.

The proposed methodology is the following recursive procedure. Consider a stratum $\X_{[h,i]}$, indexed by node $[h,i]$ and that has already been sampled according to the BSS $t$ times. It has two children in the hierarchical partitioning, namely $[h+1, 2i]$ and $[h+1, 2i+1]$. If they have been sampled a different number of times, e.g.~$T_{[h+1,2i]} < T_{[h+1,2i+1]}$, we choose the child that contains the smallest number of points, e.g.~$[h+1, 2i+1]$, and apply BSS to this child. If the number of points in each of these nodes is equal, i.e.~$T_{[h+1,2i]} = T_{[h+1,2i+1]}$, we choose uniformly at random one of these two children, and apply BSS to this child. Then we iterate the procedure in this node, until for some depth $h+l$ and node $j$, one has $T_{[h+l,j]} = 0$. Then when $T_{[h+l,j]} = 0$, sample randomly a point in stratum $\X_{[h+l,j]}$, according to $\nu_{\X_{[h+l,j]}}$. This provides the $(t+1)$th sample.

We provide in Figure~\ref{fig:sampling} the pseudo-code of this recursive procedure.
\begin{figure}[ht]
\vspace{-0.2cm}
\begin{center}
\begin{minipage}{8cm}
\bookbox{
\begin{algorithmic}
\STATE \hspace{2cm} $X = ${\bf BSS}$([p,j])$
\STATE   {\bf if} $T_{[p+1, 2j]} \neq T_{[p+1, 2j+1]}$ {\bf then}
\STATE \hspace{0.5cm} {\bf return} {\bf BSS}$\big(\arg \min(T_{[p+1,2j]}, T_{[p+1,2j+1]})\big)$
\STATE \hspace{0cm}  {\bf else if} $T_{[p+1, 2j]} = T_{[p+1, 2j+1]} >0$ {\bf then}
\STATE \hspace{0.5cm} {\bf return} {\bf BSS}$\big( [p+1, 2j + \mathcal B(1/2)\big)$
\STATE \hspace{0cm}  {\bf else}
\STATE \hspace{0.5cm} {\bf return} $X \sim \nu_{\X_{[p,j]}}$
\STATE \hspace{0cm}  {\bf endif}
\end{algorithmic}}
\end{minipage}
\vspace{-0.2cm}
\caption{Recursive BSS procedure. $\mathcal B(1/2)$ is a sample of the Bernouilli distribution of parameter $1/2$ (i.e.~we sample at random among the two children strata).}\label{fig:sampling}
\vspace{-0.8cm}
\end{center}
\end{figure}
\noindent
An immediate property is that if stratum $[h,i]$ is sampled $t$ times according to the BSS, any descendant stratum $[p,j]$ of $[h,i]$ is such that $T_{[p,j]} \geq \lfloor \frac{w_p}{w_h} t \rfloor \geq \frac{w_p}{w_h} t - 1$.

We also provide the following Lemma providing properties of an estimate of the empirical mean when sampling with the BSS.
\vspace{-0.2cm}
\begin{lemma}\label{lem:uss}
\vspace{-0.2cm}
 Let $\X_{[h,i]}$ be a stratum where one samples $t$ times according to the BSS. Then the empirical mean $\hmu_{[h,i]}$ of the samples is such that
\vspace{-0.6cm}
\begin{equation*}
 \E [\hmu_{[h,i]}] = \mu_{[h,i]},
 \quad and \quad
 \V [\hmu_{[h,i]}] \leq \frac{\si_{[h,i]}^2}{t}.
\vspace{-0.5cm}
\end{equation*}
\vspace{-0.4cm}
\end{lemma}
The proof of this Lemma is in the Supplementary Material (Appendix~\ref{proof:uss}). This Lemma also holds for the children nodes of $[h,i]$ (for a descendant $[p,j]$, it holds with $\lfloor \frac{w_p t}{w_h} \rfloor$ samples, since the procedure is recursive).

%This sampling scheme is thus efficient. It is meaningful to write the pseudo-risk on a partition where the samples in each node are collected according to the USS, since the variance of the estimate of the mean constructed with this sampling scheme is smaller than or equal to crude Monte-Carlo on the stratum.

\vspace{-0.2cm}
\paragraph{A variant of the BSS: the BSS-A procedure}

We now define a variant of the BSS: the BSS-A sampling scheme.

The reason why we need also this variant is that it is crucial, if two children of a node have obviously very different variances, to allocate more samples in the node that has higher variance. Indeed, the number of samples that one allocates to a node is directly linked to the amount of exploration that one can do of this node, and thus to the local refinement of the partitioning taht one may consider. But it is also necessary to be careful and have an allocation that is more efficient than uniform allocation, as it is not sure that it is a good idea to split the parent-node. In order to do that, we construct a scheme that uses upper confidence bounds for the less variating node, and lower confidence bounds for the most variating node: we use the $r_{[h,i]}$ that were defined for this purpose. We assume that these $r_{[h,i]}$ are defined in some sub-tree $\T^e$ of the hierarchical partitioning, and undefined outside. Using such an allocation is naturally less efficient than the optimal oracle allocation, but however more efficient than uniform allocation. We illustrate this concept in Figure~\ref{fig:samplingsch} and provide the pseudo-code in Figure~\ref{fig:sampling2}.

\begin{figure}[h]
\hspace{-1.5cm}
\includegraphics[width=11cm]{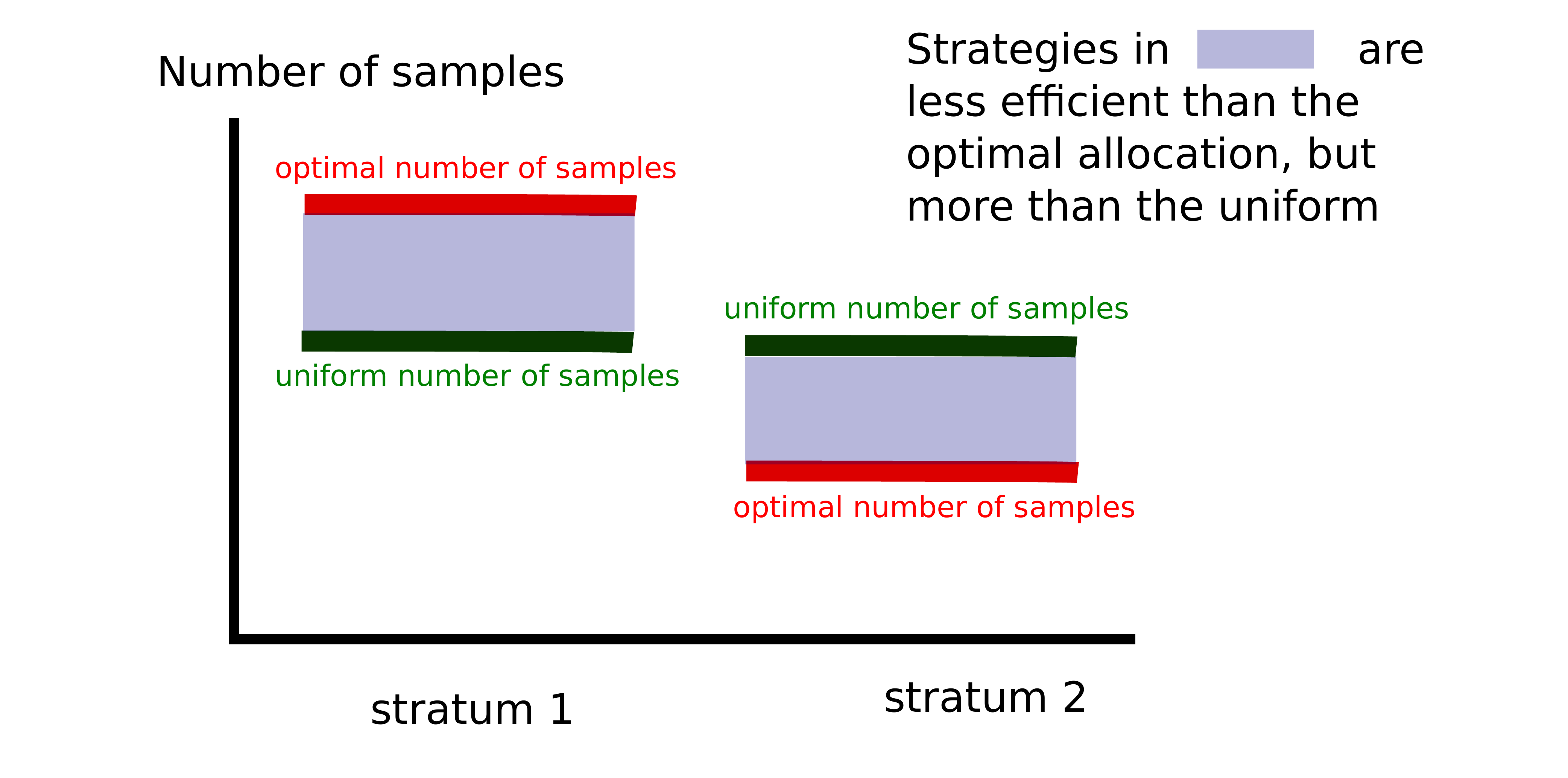}
\vspace{-1cm} \caption{With high probability, the children of each node in $\N_n$ are sampled a number of time that is in the gray zone by MC-ULCB.} \label{fig:samplingsch}
\hspace{-1cm}
\vspace{-0.5cm}
\end{figure}

\begin{figure}[h]
\vspace{-0.5cm}
\bookbox{
\begin{algorithmic}
\STATE \hspace{2.2cm} $X = ${\bf BSS-A}$([p,j], \T^e)$
\STATE \hspace{0.2cm}  {\bf if} $\{[p+1,2j], [p+1,2j+1]\} \in \T^e$ {\bf then}
\STATE \hspace{0.7cm} {\bf return} {\bf BSS-A}$\big(\arg \min(\frac{r_{[p+1,2j]}}{T_{[p+1,2j]}}, \frac{r_{[p+1,2j+1]}}{T_{[p+1,2j+1]}}), \T^e\big)$
\STATE \hspace{0.2cm}  {\bf else}
\STATE \hspace{0.7cm} {\bf return} $X = {\bf BSS}([p,j])$
\STATE \hspace{0.2cm}  {\bf endif}
\end{algorithmic}}
\vspace{-0.5cm}
\caption{Recursive BSS-A procedure.}\label{fig:sampling2}
\vspace{-0.5cm}
\end{figure}

\vspace{-0.2cm}

\subsection{Algorithm Monte-Carlo Upper-Lower Confidence Bound}\label{ss:MCULCB}

\vspace{-0.2cm}

We describe now the algorithm Monte-Carlo Upper-Lower Confidence Bound. It is decomposed in two main phases, a first Exploration Phase, and then an Exploitation Phase.

The {\bf Exploration Phase} uses Upper and Lower Confidence bounds for allocating correctly the samples. During this phase, we update an Exploration partition, that we write $\N_t^e$, and that is included in the hierarchical partitioning. When, in a stratum $[h,i] \in \N_t^e$, there are more than $t_h$ samples (also if the standard deviation of teh stratum is large enough), we update $\N_t^e$ by setting $\N_{t+1}^e = \N_t^e \bigcup [h+1,2i] \bigcup [h+1,2i+1] \setminus [h,i]$: we divide $[h,i]$ in its two children strata, and compute the $r$ corresponding to the children strata. The points are then allocated in the strata according to $\frac{r_{[h,i]}}{T_{[h,i],t}}$: a point is allocated in stratum $[h,i] \in \N_t^e$ if $\frac{r_{[h,i]}}{T_{[h,i],t}} \geq \frac{4 \tilde \Sigma}{n}$. All the points are allocated inside each stratum $[h,i] \in \N_t^e$ according to the BSS procedure.

The Exploration Phase stops at time $T$, when every node $[h,i] \in \N_T^e$ is such that $\frac{r_{[h,i]}}{T_{[h,i],T}+1} < \frac{4 \tilde \Sigma}{n}$. We write $\T_T^e$ the tree that is composed of all the nodes in $\N_T^e$ and of their ancestors. The algorithm selects in this tree a partition, that we write $\N_n$, and that is an empirical minimiser (over all partitions in $\T_T^e$) of the upper bound on the regret of algorithm MC-UCB.

Finally, we perform the {\bf Exploitation Phase} which is very similar to launching algorithm MC-UCB on $\N_n$. We pull the samples in the strata of $\N_n$ according to the BSS-A sampling scheme (described in Figure~\ref{fig:sampling2}). We compute the final estimate $\hat \mu_n$ of $\mu$ as a stratified estimate with respect to the deepest partition of $\T_T^e$, i.e.~$\N_T^e$:
\vspace{-1cm}
\begin{align}\label{eq:mea2}
\vspace{-1cm}
\hmu_n = \sum_{[h,i] \in \N_T^e} w_{h} \hmu_{[h,i],n},
\vspace{-1cm}
\end{align}
where $\hmu_{[h,i],n}$ is the empirical mean of all the samples in stratum $\X_{[h,i]}$.

We now provide the pseudo-code of algorithm MC-ULCB in Figure~\ref{f:m-algorithm-2}.
\begin{figure}[ht]
\vspace{-0.5cm}
\bookbox{
\begin{algorithmic}
\STATE \textbf{Input:} $f_{\max}$, $b$ and $\delta$.
\STATE \textbf{Initialization:}  Pull $t_{0}$ samples by BSS($[0,0]$). Set $\N_{t}^e = \{[0,0]\}$.
\STATE \textbf{Exploration Phase:}
\WHILE{$\exists [h,i] \in \N^e_t : \frac{r_{[h,i]}}{T_{[h,i]},t} > \frac{4 \tilde \Sigma}{n}$}
\STATE Take a sample in BSS($[h,i]$).
\IF{$\exists [h,i]\in \N^e_t : \Big\{ T_{[h,i],t} = 2 t_{h+1}, w_h \hsi_{[h,i],t} \geq 6Hc\sqrt{A}\frac{w_h^{2/3}}{n^{1/3}} , h < H\Big\}$}
\STATE $\N^e_{t+1} = \N^e_t \bigcup [h+1,2i] \bigcup [h+1,2i+1] \setminus [h,i]$
\STATE Compute $r_{[h+1,2i]}$ and $r_{[h+1,2i+1]}$
\ENDIF
\ENDWHILE
\STATE Select $\N_n$ such that $\N_n = \arg\min_{\N \in \T^e_n} \Big( \hat \Sigma_{\N} + (C_{\max}' - \sqrt{A}) \sum_{y \in \N} \frac{w_y^{2/3}}{n^{1/3}} \Big)$
\STATE $T=t$
\STATE {\bf Exploitation Phase:}
\FOR{$t = T+1 ,\ldots,n$}
%  \STATE Compute $\hsi_{[h,i]}$ for any $[h,i] \in \N_n$%with only the first $A w_h^{2/3} n^{2/3}$ points in leaf $[h,i] \in \N_n$
  \STATE Compute $B_{[h,i],t} = \frac{w_{h}}{T_{[h,i],t-1}} \Big( \hsi_{[h,i]} + \sqrt{\frac{A}{n^{1/3}}}  \Big)$ for any $[h,i] \in \N_n$
  \STATE Choose a leaf $[h,i]_t$ such that $[h,i]_t=\arg\max_{[p,j]\in \N_n} B_{[p,j],t}\quad$
  \STATE Pick a point according to BSS-A($[h,i]_t$)
\ENDFOR
\STATE \textbf{Output:} $\hmu_{n}$
\end{algorithmic}}
\vspace{-0.5cm}
\caption{The pseudo-code of the Tree-MC-UCB algorithm. The empirical standard deviations and means $\hsi_{[h,i]}$ and $\hmu_{n}$ and $\tilde \si_{[h,i]}$ are computed using Equation~\ref{eq:estim-var2},~\ref{eq:mea2} and~\ref{eq:estim-var22}. The value of $r_{[h,i]}$ is computed using Equation~\ref{eq:r}. The BSS algorithm is described in Figure~\ref{fig:sampling} and the BSS-A algorithm is described in Figure~\ref{fig:sampling2}.}\label{f:m-algorithm-2}
\vspace{-0.5cm}
\end{figure}

\vspace{-0.2cm}

\subsection{Main result}
\vspace{-0.2cm}

We are now going to provide the main result for the pseudo-risk of algorithm MC-ULCB.
\vspace{-0.1cm}
\begin{theorem}\label{th:algo.2}
Under Assumption~\ref{ass:stratummeasurables} and~\ref{ass:stratumequal} for the strata and~\ref{ass:boundness} for the function $F$, the pseudo-risk of algorithm MC-ULCB is bounded with probability $1-\delta$ as
\begin{small}
\begin{align*}
% L_n = \sum_{x \in \N_n^e} \frac{(w_x \si_x)^2}{T_{x,n}} \leq \frac{\Sigma_{\N_n}^2}{n} + B \Sigma_{\N_n} \sum_{y\in{\N_n}}\frac{w_y^{2/3}}{n^{1/3}} \leq \min_{\N} \Bigg[\frac{\Sigma_{\N}^2}{n} + C_{\max}' \Sigma_{\N_n} \sum_{y\in{\N}}\frac{w_y^{2/3}}{n^{1/3}} \Bigg],
 &L_n(\alg_{MC-ULCB}) \leq  \sum_{[h,i] \in \N_n} \frac{(w_h \si_{[h,i]})^2}{T_{[h,i],n}}\\
&\leq \min_{\N} \Bigg[\frac{\Sigma_{\N}^2}{n} + C_{\max}' \Sigma_{\N} \sum_{[h,i]\in{\N}}\frac{w_h^{2/3}}{n^{4/3}} + C_{\max}'^2 \Big(\sum_{[h,i]\in{\N}}\frac{w_h^{2/3}}{n^{4/3}}\Big)^2
\Bigg],
\end{align*} 
\end{small}
where $\min$ means minimum over all partitions of the hierarchical partitioning, and $C_{\max}' \leq 320\sqrt{(1 + 3b + 4f_{\max})\log(4n^2(3f_{\max})^3/\delta)} (1/\sigma_{[0,0]} + 1)(8\sigma_{[0,0]} + 1)\log\big((3 f_{\max})^3 n\big)$.
\end{theorem}

The proof of this result is in the Supplementary Material (Appendix~\ref{app:MC-ULCB}).

A first remark on this result is that even the first inequality (i.e.~$L_n(\alg_{MC-ULCB}) \leq  \sum_{[h,i] \in \N_n} \frac{(w_h \si_{[h,i]})^2}{T_{x,n}}$) is not trivial since the algorithm does not sample at random according to $\nu_{\X_{[h,i]}}$ in the strata $[h,i] \in \N_n$, but according the BSS-A. It was necessary to do that since in order to select wisely $\N_n$, one should have explored the tree $\T_T^e$, and thus it was necessary to allocate the points in order to allow splitting of the nodes and adequate exploration.

Assume that $\min_{\mathcal N} \Sigma_{\N}$ is lower bounded, e.g.~the function $F$ is noisy (i.e.~the function $s$ is not almost surely equal to $0$). Then a second remark is that the second term in the final bound, namely $C_{\max}'^2 \Big(\sum_{[h,i]\in{\N}}\frac{w_h^{2/3}}{n^{4/3}}\Big)^2$, is negligible when compared to the second term, namely $\sum_{[h,i]\in{\N}}\frac{w_h^{2/3}}{n^{4/3}}$. Indeed, since $\si_{[0,0]}$ is bounded by Assumption~\ref{ass:boundness} by $f_{\max}$, we know that $\min_{\N} \Bigg[\frac{\Sigma_{\N}^2}{n} + C_{\max}' \Sigma_{\N} \sum_{[h,i]\in{\N}}\frac{w_h^{2/3}}{n^{4/3}}\Bigg]$ is smaller than $\frac{C_{\max}'f_{\max}+f_{\max}^2}{n}$, which implies that for one of the partitions $\N$ that realises this minimum, we have $C_{\max}'^2 \Big(\sum_{[h,i]\in{\N}}\frac{w_h^{2/3}}{n^{4/3}} \leq C_{\max}'^2 \frac{(C_{\max}'f_{\max}+f_{\max}^2)^2}{n^2}$, which is negligible when compared to $n^{-4/3}$ and thus in particular $\sum_{[h,i]\in{\N}}\frac{w_h^{2/3}}{n^{4/3}}$.

\vspace{-0.2cm}
\subsection{Discussion}\label{ss:disc}
\vspace{-0.2cm}

{\bf Algorithm MC-ULCB does almost as well as MC-UCB on the best partition:} The result in Theorem~\ref{th:algo.2} states that algorithm MC-ULCB selects adaptively a partition that is almost a minimiser of the upper bound on the pseudo-risk of algorithm MC-UCB. It then allocates almost optimally the samples in this partition. Its upper bound on the regret is thus smaller, up to additional multiplicative term contained in $C_{\max}'$, than the upper bound on the regret of algorithm MC-UCB launched on an optimal partition of the hierarchical partitioning. The issue is that $C_{\max}'$ is bigger than the constant $C_{\min}$ for MC-UCB. More precisely, we have $C_{\max}' = C_{\min} \times C\log\big((3 f_{\max})^3n \big)$, where $C$ is a constant depending of $f_{\max}$ and $b$ (see bound on $C_{\max}'$ in Theorem~\ref{th:algo.2}). This additional dependency in $\log(n)$ is not an artifact of the proof and appears since we perform some model selection for selecting the partition $\N_n$. We do not know whether it is possible or not to get rid of it. Note however that a $\log$ factors already appears in the bound of MC-UCB, and that the question of whether it is or not needed remains open.
%and that $\Sigma_{\N_n}$ appears instead of $\Sigma_{\N}$. It is smaller than the constant $C_{\max}$ in Theorem~\ref{th:algo.1} for algorithm Deep-MC-UCB. We believe that it is very hard (impossible?) to manage to have the same constant in MC-UCB and MC-ULCB for the best partition.

{\bf The final partition $\N_n$:} Algorithm MC-ULCB refines more the partition $\N_n$ in parts of the domain where splitting a stratum $[h,i]$ in a sub-partition $\B_{[h,i],\N}$ is such that $w_{[h,i]} \si_{[h,i]} - \sum_{x \in \B_{[h,i],\N}} w_x \si_x$ is large. Note that this corresponds, by definition of the $\si_{[h,i]}$, to parts of the domain where $g$ and $s$ have large variations. We do not refine the partition in regions of the domain where this is not the case, since it is more efficient to have also as few strata as possible.

{\bf The sampling schemes:} The key-points in this paper are the sampling schemes. Indeed, we construct and use a sampling technique, the BSS, that is such that the samples are collected in a way that reminds low discrepancy sampling schemes\footnote{Although the samples are chosen randomly, the sampling scheme is such that we know in a deterministic and exact way the number of samples in each not too small part of the domain.} on the domain, and provide an estimate such that its variance is smaller than the one of crude Monte-Carlo. We also build another sampling scheme, BSS-A. This sampling scheme ensures that, with high probability, if two children strata have very different variances, then the one with higher variance is more sampled. At the same time, it ensures that if finally the decision of splitting a stratum is not taken, then the allocation in the stratum is still better than or as efficient as random allocation according to $\nu$ restricted to the stratum.

{\bf Evaluation of the precision of the estimate and confidence intervals:} An important question that one can ask here is on the prssibility of constructing a confidence interval around the estimate that we obtain. What we would suggest in this case is to upper bound the pseudo-risk of the estimate by $(\sum_{x \in \N_n} (w_x \hat \sigma_x + w_x^{2/3}/n^{1/3}))^2/n$, and construct a confidence interval considering this as a bound on the variance or the estimate, using e.g. Bennett's inequality. If e.g. the noise is symmetric, then the pseudo-risk equals the mean squared error, and the confidence interval is valid, and in particular asymptotically valid (see~\cite{rapp-tech-MC-UCB}). Also it is less wide (up to a negligible term) than the smallest valid confidence interval on the best (oracle) stratified estimate on the hierarchical partitioning (and then in particular than the one for the crude MC estimate). Indeed, the oracle variance of such estimate is $(\inf_{N} \sum_{x \in \N} w_x \sigma_x)^2/n$ which is by definition of $N_n$ larger or equal up to a negligible term to $(\sum_{x \in \N_n} w_x \sigma_x)^2/n$, and this equals up to a negligible term to the upper bound on the pseudo-risk we used to construct the confidence interval.

\section*{Conclusion}

In this paper, we presented an algorithm, MC-ULCB, that aims at integrating a function in an efficient way.

MC-ULCB improves the performances of Deep-MC-UCB and returns an estimate whose pseudo-risk is smaller, up to a constant, than the minimal pseudo-risk of MC-UCB run on any partition of the hierarchical partitioning. The algorithm adapts the partition to the function and noise on it, i.e.~it refines more the domain where $g$ and $s$ have large variations. We believe that this result is interesting since the class of hierarchical partitioning is very rich and can approximate many partition.

\paragraph{Acknoledgements:}The research leading to these results has received funding from the European Community's Seventh Framework Programme (FP7/2007-2013) under grant agreement n° 270327.

\newpage

\nocite{langley00}

\bibliography{allocation}
\bibliographystyle{icml2013}

%%%%%%%%%%%%%%%%%%%%%%%%%%%%%%%%%%%%%%%%%%%%%%%%%%%%%%%%%%%%%%%%%%%%%%%%%%%%%%%
%% APPENDIX
%%%%%%%%%%%%%%%%%%%%%%%%%%%%%%%%%%%%%%%%%%%%%%%%%%%%%%%%%%%%%%%%%%%%%%%%%%%%%%%
\newpage
\appendix

\onecolumn[

%%%%%%%%%%%%%%%%%%%%%%%%%%%%%%%%%%%%%%%%%%%%%%%%%%%%%%%%%%%%%%%%%%%%%%%%%%%%%%%
%% PROOF BERNSTEIN
%%%%%%%%%%%%%%%%%%%%%%%%%%%%%%%%%%%%%%%%%%%%%%%%%%%%%%%%%%%%%%%%%%%%%%%%%%%%%%%

{\huge Supplementary Material for paper: "Toward Optimal Stratification for Stratified Monte-Carlo Integration"}

We first introduce the following natation. We write $\B_{[h,i], \N}$, where $\N$ is a cut of a dyadic tree, the sub-partition given by the leafs of the tree issued from $[h,i]$  and with leaves $\N$ (we branch partition $\N$ on leaves $[h,i]$). We illustrate this in Figure~\ref{fig:branching}.
\begin{figure}[htbp]
\begin{center}
\includegraphics[width=6.5cm]{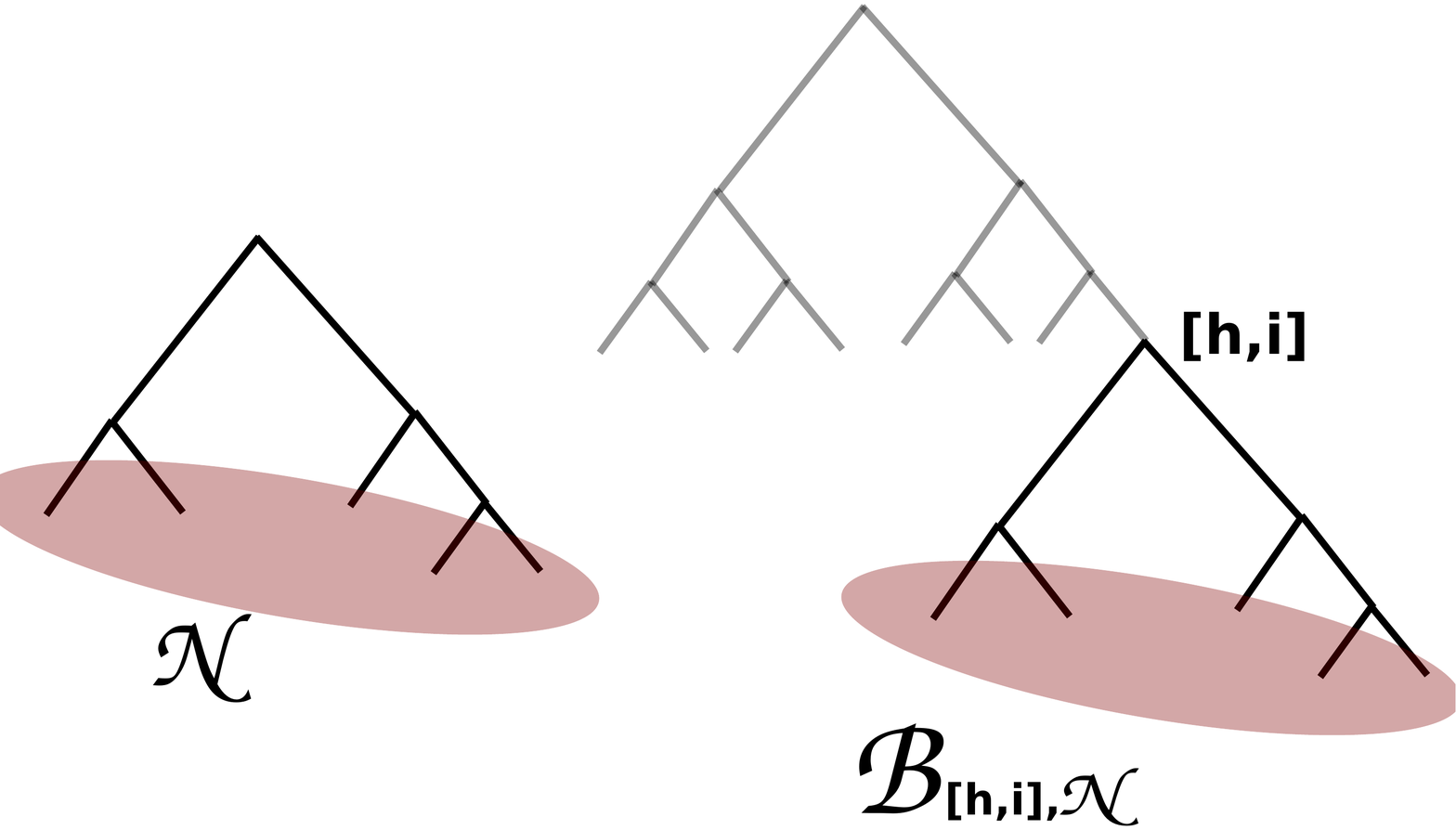}
\vspace{-0.6cm}
%\end{minipage}
\caption{Illustration of $\B_{[h,i], \N}$.} \label{fig:branching}
\vspace{-0.4cm}
\end{center}
\end{figure}
Similarly and by a slight abuse of notations, we write for any integer $l>0$ the sub-tree $\B_{[h,i], l}$ as the sub-tree ussyed from node $[h,i]$ and extended until depth $h+l$.

\vspace{-0.2cm}

\section{Numerical experiments}\label{s:experiments}

\vspace{-0.2cm}
%%%%%%%%%%%%%%%%%%%%%%%%%%%%%%%%%%%%%%%%%%%%%%%%%%%%%%%%%%%%%%%%%%%%%%%%%%%%%%%
%% Comparison
%%%%%%%%%%%%%%%%%%%%%%%%%%%%%%%%%%%%%%%%%%%%%%%%%%%%%%%%%%%%%%%%%%%%%%%%%%%%%%%

We consider the pricing problem of an Asian option introduced in \cite{glasserman1999asymptotically} and later considered in \cite{kawai2010asymptotically,A-EtoJou10}.
This uses a Black-Scholes model with strike $C$ and maturity $T$. Let $(W_t)_{0\leq t\leq T}$ be a Brownian motion. The discounted payoff of the Asian option is defined as a function of $W$, by:
\vspace{-0.2cm}
\begin{footnotesize}
\begin{equation*}
F((W_t)_t) = e^{-rT} \max \Big[ \int_0^T S_0 e^{\Big( (r - \frac{1}{2}s_0^2)t + s_0 W_t\Big)} dt -C , 0 \Big],
\vspace{-0.2cm}
\end{equation*}
\end{footnotesize}
%\vspace{-0.1cm}
where $S_0$, $r$, and $s_0$ are constants.

We want to estimate the price $p = \E_{W}[F(W)]$ by Monte-Carlo simulations (by sampling on $W$). In order to reduce the variance of the estimated price, we stratify as suggested in~\cite{glasserman1999asymptotically,kawai2010asymptotically} the space of $(W_t)_{0\leq t\leq T}$ according to the quantiles of $W_T$, i.e.~the quantiles of a normal distribution $\mathcal N(0,T)$. In other words, we re-write $F:=F((W_t)_{0 \leq t < T}, x)$ where $x \in \X=[0,1]$ is the quantile that corresponds to $W_T$. In this context, the noise $\epsilon$ comes from the directions along which we do not stratify, namely $(W_t)_{0 \leq t < T}$. After having sampled $W_T$ according to the algorithm for stratified Monte-Carlo (e.g.~MC-ULCB), we simulate the rest of the Brownian motion $(W_t)_{0 \leq t < T}$ by a Brownian Bridge (concretely, we discretize this Brownian motion in order to be able to simulate it in $16$ values).
We choose the same numerical values as \cite{kawai2010asymptotically}: $S_0=100$, $r=0.05$, $s_0=0.30$, $T=1$ and $d=16$. We choose a strike $C = 90$. %Note that the strike $C$ of the option has a direct impact on the variability of the strata. Indeed, the larger $C$, the more probable $F(W)=0$ for strata with small $W_d$, and thus, the smaller $\lambda_{\min}$.

%
% \begin{figure*}[!hbtp]
% \begin{minipage}{0.3\textwidth}
% \includegraphics[width=\textwidth]{K60.eps}
% \end{minipage}
% \begin{minipage}{0.3\textwidth}
% \includegraphics[width=\textwidth]{K90.eps}
% \end{minipage}
% \begin{minipage}{0.3\textwidth}
% \includegraphics[width=\textwidth]{K120.eps}
% \end{minipage}
% \caption{Representation of $\E(F((W_t)_t)/W_d=x)$ in function of $\Prob(W_d \leq x)$ for different values of the strike $C$.} \label{fig:diffvalC}
% \end{figure*}

By studying the range of the $F(W)$, we set the (meta-)parameters of the algorithm MC-ULCB to $A= 2\log(n)$ and $H = 0.3 \log(n)$ (the other parameters adjust automatically with these two meta-parameters). Our main competitor is the algorithm described in~\cite{A-EtoForJouMou11}, to which we refer to as A-SSAA, and which also perform adaptive allocation and stratification. %In order to illustrate the performances of these algorithms, we compare them with MC-UCB launched on partitions increasing refinement.

We first observe the behaviour of MC-ULCB with a budget of $n=2000$. On a typical run, algorithm MC-ULCB divides the domain $[0,1]$ in approximately $15$ strata that form partition $\N_n$, and the partition is more refined where $s$ and $g$ vary more. We illustrate this in Figure~\ref{fig:diffvalC3}.
\begin{figure}[!hbtp]
\vspace{-0.5cm}
\begin{center}
\includegraphics[width =8cm, height=6cm]{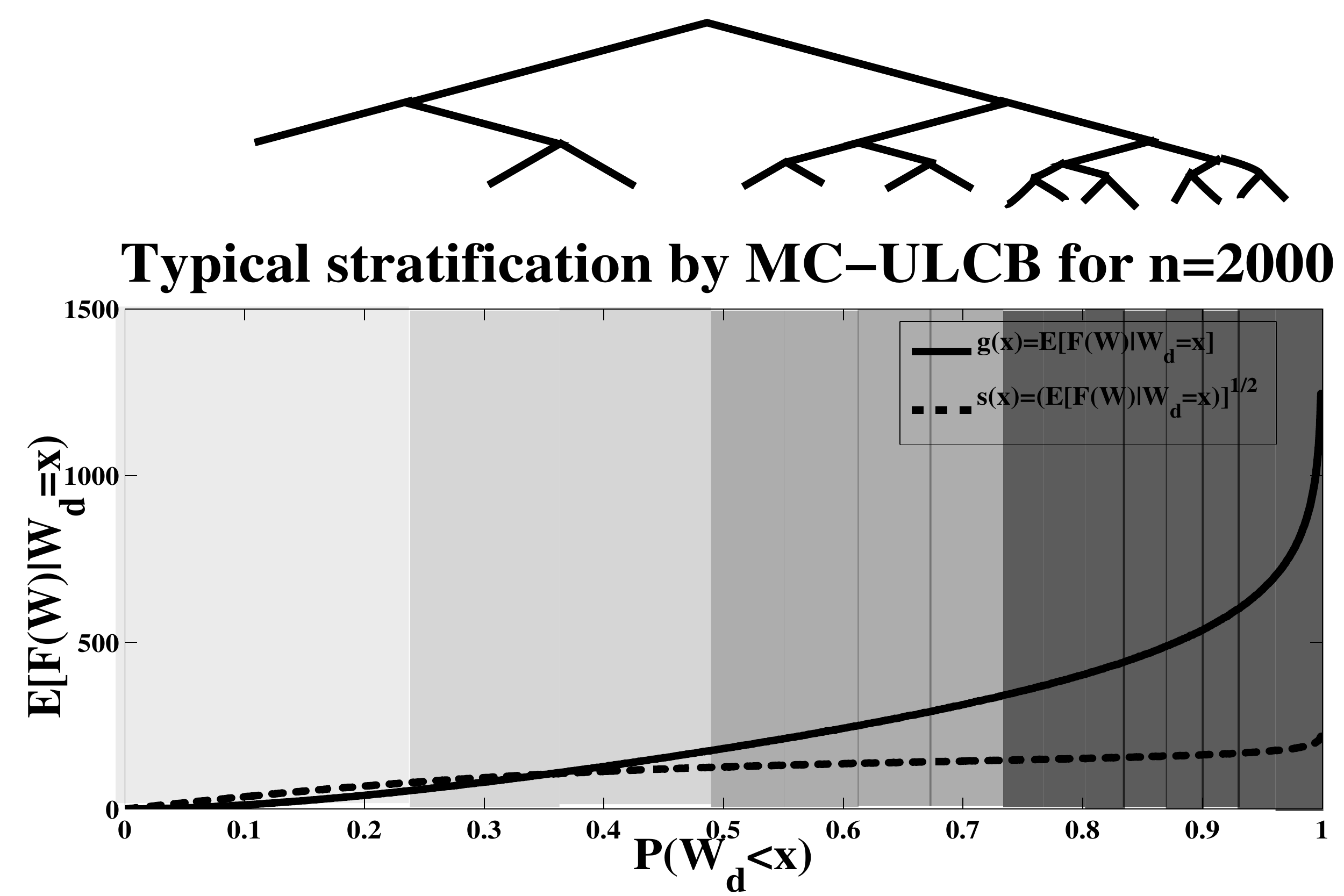}
\end{center}
\vspace{-0.7cm}
\caption{Stratification of the space for a typical run of MC-ULCB.} \label{fig:diffvalC3}
\vspace{-0.5cm}
\end{figure}

In Figure~\ref{f:perf100000}, we display the (averaged over $10000$ runs) performances of algorithms MC-ULCB, A-SSAA, and MC-UCB (launched on some partitions in $K$ hypercubes of same measure). Note first that trough the performances of MC-UCB launched on partitions with varying number of strata, we observe the optimal number of strata increases with $n$. We observe that MC-ULCB is more efficient than algorithm MC-UCB launched on any of these partitions in $K$ strata. This is not very surprising since we only consider MC-UCB launched on partitions where all strata have the same size, i.e.~these partitions are not adapted to the function $F$. We would probably observe slightly better results for MC-UCB if we launched it on an oracle partition with respect to $F$, but such a partition is not easy to build, even when the function $F$ is known. Also, MC-ULCB is more efficient than A-SSAA, and that for any sample size. It is not very surprising since the price model for Asian option happens to verify Assumption~\ref{ass:boundness}, which is more restrictive than the assumptions made in paper~\cite{A-EtoForJouMou11}. This Assumption is used to tune the algorithm. In paper~\cite{A-EtoForJouMou11}, since they do not make this sub-Gaussian assumption, they can not calibrate the length of the exploration phase with respect to the properties of the distribution, and thus fit the exploration/exploitation to the problem.

\begin{table}[ht]
\begin{center}
\begin{tabular}{|c||c|c|c|c|}
  \hline
Budget $n$ & $n=200$  & $n=2000$  & $n=20000$ \\
\hline
\hline
  Crude MC & $5.1$ & $5.1\; 10^{-1}$ & $5.1 \; 10^{-2}$ \\%& $1.02 \; 10^{-2}$ \\
  \hline
  MC-UCB, $K=5$ & $4.65$ & $4.65 \;10^{-1}$ & $4.64 \; 10^{-2}$ \\%& $9.29 \;10^{-3}$\\
  \hline
  MC-UCB, $K=10$ & $4.56$ & $4.55 \;10^{-1}$ & $4.55 \; 10^{-2}$ \\%& $9.10 \;10^{-3}$\\
\hline
  MC-UCB, $K=20$ & $4.63$ & $4.49 \;10^{-1}$ & $4.41 \; 10^{-2}$ \\%& $8.81 \;10^{-3}$ \\
\hline
  MC-UCB, $K=40$ & $4.71$ & $4.655 \;10^{-1}$ & $4.31 \; 10^{-2}$ \\%& $8.36 \;10^{-3}$\\
\hline
  A-SSAA & $4.32$ & $4.25 \;10^{-1}$ & $4.13 \; 10^{-2}$ \\%& $8.17 \;10^{-3}$\\
\hline
  MC-ULCB & $4.08$ & $3.95 \;10^{-1}$ & $3.82 \; 10^{-2}$ \\%& $7.78 \;10^{-3}$
\hline
\end{tabular}
\label{f:perf100000}
\end{center}
\caption{Mean squared errors of the estimates outputted by the strategies for different values of $n$.}
\end{table}

\section{Proof of Lemma~\ref{lem:uss}}\label{proof:uss}

Assume that stratum $\X_{[h,i]}$ has been sampled $t$ times according to the BSS. Let $(A_0,\ldots, A_l) \in \{0,1\}^l$ be the (uniquely defined) decomposition in basis $2$ of $t$, i.e.~$\sum_{p=0}^l A_p 2^r = t$ and $A_l = 1$. This implies by Assumption~\ref{ass:stratumequal} and by definition of $(A_r)_r$, that  $\sum_{p=0}^l A_p\frac{ w_h}{w_p} = t$. We denote by $\mathcal D_l = (X_1,\ldots, X_t)$ the set of the $t$ samples in stratum $\X_{[h,i]}$.

%Let $l = \lfloor \log(t) \rfloor$ be the maximal depth such that there is at least one point per sub-stratum in $\B_{[h,i], l}$.
By construction of the BSS, there are at most two and at least one element of $\mathcal D_l$ in each stratum of $\B_{[h,i], l}$.
For all $j \leq 2^{h+l} -1$, we write $X_{l,j}$ the first sample in stratum $[h+l,j]$. Conditionally to the number $t$ of samples, each of these samples is pulled randomly in stratum $[h+l,j]$ according to $\nu_{\X_{[h+l,j]}}$.

%Let us consider a sub-stratum in $\B_{[h,i], l}$ that contains two samples. 
Let us now consider the largest $p <l$ such that $A_p = 1$. Let us consider $\mathcal D_p = \mathcal D_l \setminus \{(X_{l,j})_{[h+l,j] \in \B_{[h,i], l} }\}$.
By construction of the BSS, conditionally to the knowledge that there is a re-numeration of the samples such that $\forall 0 \leq j < 2^l, X_{l,j} \sim \nu_{\X_{[h+l,j]}}$ (and thus conditionally \textit{only} to the number $t$ of samples since the fact that there is a re-numeration such that $\forall 0 \leq j < 2^l, X_{l,j} \sim \nu_{[h+l,j]}$ follows deterministically from the budget $t$), there are at most two and at least one element of $\mathcal D_p$ in each stratum of $\B_{[h,i], p}$.
%the second sample in one of the strata $\B_{[h,i], l}$ is such that this second sample has been pulled randomly in the stratum $[h+p,j]$ (that contains the second sample). 
We note $X_{p,j}$ the first sample. By construction of the BSS and conditionally to the number $t$ of samples, each of these samples is pulled randomly in stratum $[h+p,j]$ according to $\nu_{\X_{[h+p,j]}}$.

We can continue this induction for every $p$ such that $A_p = 1$. We have, at the end of the induction, relabeled (trough the relabeling that we presented) every sample (in $\mathcal D_l$) by $X_{p,j}$. We know that conditional to the number $t$ of samples, $\forall p/ A_p = 1$, and $\forall 0\leq j \leq 2^{h+p} -1$, $X_{p,j} \sim \nu_{\X_{[p,j]}}$ and also that these relabeled samples are all independent of each other (although the relabeling of each sample is random and is not independent of the other samples).
%Note that by construction of the BSS, this $p$ is precisely such that $A_p =1$ (since the BSS always pull at random a sample in regions of the domains where there are the least points).
%Let $A_p$ be the indicator that takes as value $1$ if there are samples in the partition $\B_{[h,i], p}$.

The empirical mean $\hat \mu_{[h,i]}$ on stratum $[h,i]$ thus satisfies
\begin{equation*}
 \hat \mu_{[h,i]} = \frac{1}{t} \sum_{s=1}^t X_s = \sum_{p=0}^l \frac{w_h}{w_pt} \sum_{[h+p,j] \in \B_{[h,i], p}} \frac{w_p}{w_h} X_{p,j} A_p.
\end{equation*}
%By construction, all the $X_{p,j} \sim \nu_{\X_{[h+p,j]}}$. 
Since by construction $\sum_{p=0}^l \frac{A_p w_h}{w_p} = t$, the empirical estimate of the mean thus satisfies
\begin{equation*}
 \E [\hat \mu_{[h,i]}] = \sum_{p=0}^l \frac{w_h}{w_pt} \sum_{[h+p,j] \in \B_{[h,i], p}}  \frac{w_p}{w_h} \mu_{[h+p,j]} A_p =  \sum_{p=0}^l \frac{w_h}{w_pt} \mu_{[h,i]} A_p = \mu_{[h,i]}.
\end{equation*}
Note now that the variance of this estimate is such that
\begin{equation*}
 \V [\hat \mu_{[h,i]}] = \sum_{p=0}^l \frac{w_h^2}{w_p^2 t^2} \sum_{[h+p,j] \in \B_{[h,i], p}}  (\frac{w_p}{w_h})^2 \si_{[h+p,j]}^2 A_p \leq  \sum_{p=0}^l \frac{w_p}{w_ht^2} \si_{[h,i]}^2 A_p  \leq \frac{\si_{[h,i]}^2}{t}.
\end{equation*}
% This USS thus provides an estimate such that
% \begin{itemize}
%  \item for a child strata $[p,j]$ of $[h,i]$, $T_{[p,j]} \geq \frac{w_p}{w_h} t - 1$.
% \item it is unbiased and its variance is smaller than the one of crude Monte-Carlo.
% \end{itemize}

\section{Preliminary results}\label{app:deepmcucb}

\subsection{An interesting large probability event}

\begin{lemma}\label{lem:xi}
For a stratum $\X_{[h,i]}$ of the hierarchical partition, write $\Big(X_{[h,i],0}, \ldots, X_{[h,i],n}\Big)$ the samples collected by BSS in stratum $\X_{[h,i]}$ (or by BSS in a stratum of smaller depth).
Consider the event
 \begin{equation}\label{def:xi}
 \xi = \bigcap_{[h,i] : h \leq H}  \bigcap_{t=2}^n \Bigg\{ \Big|\sqrt{\frac{1}{2^{\lfloor \log(t) \rfloor}}\sum_{a=0}^{2^{\lfloor \log(t) \rfloor} - 1} \Big(X_{[h,i],a} - \frac{1}{2^{\lfloor \log(t) \rfloor}}\sum_{a'=0}^{2^{\lfloor \log(t) \rfloor}} X_{[h,i], a'} \Big)^2 } - \si_{[h,i]}\Big| \leq A\sqrt{\frac{1}{t}} \Bigg\},
\end{equation}
where $A = 2\sqrt{2(1 + 3b + 4f_{\max})\log(4n^2(3f_{\max})^3/\delta)}$ and $H = \lfloor \frac{\log\big((3 f_{\max})^3 n\big)}{\log(2)} \rfloor +1$. Then $\P(\xi) \geq 1-\delta$.

Note also that for $h \geq H, \forall i \leq 2^h-1$, we have
\begin{align*}
 w_{[h,i]}\sigma_{[h,i]} \leq \frac{w_{[h,i]}^{2/3}}{n^{1/3}}.
\end{align*}
\end{lemma}
\begin{proof}
{\bf Probability of the event $\xi$}

Let $[h,i]$ be a stratum of the hierarchical partitioning such that $h \leq H$ and $t\geq 2$. Let $l= \lfloor \log(t) \rfloor$. By definition of the BSS, we know that for $s \leq 2^l$, sample $X_{[h,i],s}$, conditionally to the $s-1$ other samples, is sampled uniformly inside the strata $\X_{[h+l,k]}$ that contain no samples, and independent of the other samples.

%Consider a leaf $[h,i]$ where there are $T$ points sampled as exposed in the sampling scheme. Note that by the definition of the USS, there are $t = 2^{\lfloor \log(T) \rfloor}$ points sampled uniformly in $t$ strata of equal measures, the remaining points being sampled uniformly in these strata. Note that $t \geq \frac{T}{2}$.

%We denote by $(X_{[h,i],1},\ldots,X_{[h_t,i],T_{[h,i],t}})$ the points in node $[h,i]$.

% Using the results from \ref{eq:bernfinal} from Appendix~\ref{app:ldev}, we have with probability $1-2\delta$ that
% \begin{align*}
% |\hmu_{[h,i],t} - \mu_{[h,i]}| &= \Big|\frac{1}{T_{[h,i],t}}\sum_{a=1}^{T_{[h,i],t}} X_{[h,i],a} - \mu_{[h,i]}\Big|\\
%  &= \Big|\frac{1}{T_{[h,i],t}} \sum_{a=1}^{T_{[h,i],t}} X_{[h,i],a} - \mu_{[h,i]}\Big|\\
% &\leq \sqrt{\frac{2\bar{V}\log(2/\delta)}{T_{[h,i],t}}} + \frac{b\log(2/\delta)}{T_{[h,i],t}}.
% \end{align*}

Using the results from Lemma~\ref{ss:variance}, we know that with probability $1-\delta$, the estimate of the standard deviation computed with the $2^l$ first samples satisfies
\begin{align*}
\Big|\sqrt{\frac{1}{2^l}\sum_{a=0}^{2^l -1} \Big(X_{[h,i],a} - \frac{1}{2^l}\sum_{b=0}^{2^l-1} X_{[h,i],b}\Big)^2} - \si_{[h,i]}\Big|
&\leq 2\sqrt{\frac{(1 + 3b + 4\bar{V})\log(2/\delta)}{2^l}}\\
&\leq 2\sqrt{\frac{2(1 + 3b + 4\bar{V})\log(2/\delta)}{t}}\\
&\leq 2\sqrt{\frac{2(1 + 3b + 4f_{\max})\log(2/\delta)}{t}}.
\end{align*}

By the definition of $H$, we know that there are less than $2 \times 2^{H}$ strata in the hierarchical partitioning of depth smaller than $H$. Because of the definition of $A$, we have $\P(\xi) \geq 1-\delta$.

{\bf Characterisation of the strata of depth bigger than $H$}

Consider a node $[h,i]$ of depth $h \geq H$. As both $m$ and $s$ are bounded by $f_{\max}$ (see Assumption~\ref{ass:boundness}), then
\begin{align*}
 w_{[h,i]}\sigma_{[h,i]}
&= \sqrt{w_{h,i}} \sqrt{\int_{\X_{[h,i]}} s^2(x)dx} + \sqrt{w_{h,i}} \sqrt{\int_{\X_{[h,i]}} (g(x) - \mu_{[h,i]})^2 dx}\\
 &\leq \sqrt{w_{[h,i]}} \sqrt{\int_{\X_{[h,i]}} f_{\max}^2 dx} + \sqrt{w_{[h,i]}} \sqrt{\int_{\X_{[h,i]}} 4f_{\max}^2 dx}\\
&\leq 3 w_{[h,i]} f_{\max}.
\end{align*}

As $h \geq H$, we have $w_{[h,i]} \leq \big(\frac{1}{2}\big)^{H} \leq  \big(\frac{1}{3f_{\max}}\big)^3 \frac{1}{n}$. From that we deduce that for $h \geq H$,
\begin{align*}
 w_{[h,i]}\sigma_{[h,i]} \leq \frac{w_{[h,i]}^{2/3}}{n^{1/3}}.
\end{align*}
% 
% 
% In order for this node to be such that $w_{[h,i]} \si_{[h,i]} \geq \frac{w_{[h,i]}^{2/3}}{n^{1/3}}$ (and note that as $A \geq 1$, we also have $C_{\min} \geq 1)$, it is necessary that $3 w_{[h,i]} f_{\max} \geq  \frac{w_{[h,i]}^{2/3}}{n^{1/3}}$, which is equivalent to
% \begin{align*}
% w_{[h,i]}  \geq \Big(\frac{1}{3f_{\max}}\Big)^{3} \frac{1}{n} \geq \frac{H}{n},
% \end{align*}
% with $H=\Big(\frac{1}{3f_{\max}}\Big)^{3}$. 
% 
% There are thus less than $\frac{2n}{H}$ nodes that are such that $w_{[h,i]} \geq \frac{H}{n}$. Because of the definition of $A$, we have $\P(\xi) \geq 1-\delta$.
\end{proof}

\subsection{Rate for the algorithm MC-UCB}\label{proof:algo}

We first prove the following result.
\begin{proposition}\label{thm:m2-regret}
Let Assumption~\ref{ass:stratummeasurables}, \ref{ass:stratumequal}, and~\ref{ass:boundness} hold. Assume that $n \geq 2B \sum_{q\in \N_n} w_q^{2/3}n^{2/3}$ (with $B=\frac{\big(4\sqrt{2A} + \Sigma_{\N_n} A \big)}{\Sigma_{\N_n}}$). For any $0<\delta \leq 1$, the algorithm MC-UCB on a partition $\N_n$ satisfies on $\xi$, and thus with probability at least $1-\delta$,
\begin{equation*}
\frac{w_p\si_{p}}{T_{p,n}} \leq \frac{\Sigma_{\N_n}}{n} + \big(4\sqrt{2A} + \Sigma_{\N_n} A \big) \frac{\sum_{q\in \N_n} w_q^{2/3}}{n^{4/3}} \leq \frac{\Sigma_{\N_n}}{n} + C_{\min} \frac{\sum_{q\in \N_n} w_q^{2/3}}{n^{4/3}},
\end{equation*}
where $C_{\min} = \big(4\sqrt{2A} + \Sigma_{\N_n} A \big) $  and
\begin{equation*}
T_{p,n} \geq \lambda_{p,\Sigma_{\N_n}} \Big(n - B\big(\sum_{q\in \N_n} w_q^{1/3}\big)n^{2/3}\Big),
\end{equation*}
where $B =  \frac{\big(4\sqrt{2A} + \Sigma_{\N_n} A \big)}{\Sigma_{\N_n}}$.
\end{proposition}
\begin{proof}

\noindent
{\bf Step~1.~Properties of the algorithm.} For a node $q \in \N_{t+1}$, we first recall the definition of $B_{q,t+1}$ used in the MC-UCB algorithm
\begin{equation*}
B_{q,t+1} = \frac{w_q}{T_{q,t}} \Bigg( \hsi_{q} + \sqrt{A}\frac{1}{w_q^{1/3} n^{1/3}}\Bigg).
\end{equation*}

Using the definition of $\xi$ and the fact that if node $q$ is in $\N_{t+1}$, then $T_{q,t+1} \geq \lfloor A w_q^{2/3} n^{2/3} \rfloor$, it follows that, on $\xi$
\begin{equation}\label{e:bound.B.2}
 \frac{w_q\si_q}{T_{q,t}} \leq B_{q,t+1} \leq \frac{w_q}{T_{q,t}}\Bigg(\si_{q} + 2\sqrt{A}\frac{1}{w_q^{1/3} n^{1/3}} \Bigg).
\end{equation}
Let $t+1 \geq 2K+1$ be the time at which an arm $q$ is pulled for the last time, that is $T_{q,t} = T_{q,n}-1$. Note that there is at least one arm such that this happens as $n \geq 4K$. Since at $t+1$ arm $q$ is chosen, then for any other arm $p$, we have
\begin{equation}\label{e:meca2-2}
B_{p,t+1} \leq B_{q,t+1}\;.
\end{equation}
From Equation~\ref{e:bound.B} and $T_{q,t} = T_{q,n} -1$, and also since by construction of the algorithm $T_{q,n} \geq 2$, we obtain on $\xi$
\begin{equation}\label{e:meca.lb2.2}
B_{q,t+1} \leq \frac{w_q}{T_{q,t}}\Bigg(\si_{q} + 2\sqrt{2A}\frac{1}{w_q^{1/3} n^{1/3}} \Bigg).
\end{equation}
Furthermore, since $T_{p,t} \leq T_{p,n}$, then on $\xi$
\begin{equation}\label{e:meca.hb2-2}
B_{p,t+1} \geq \frac{w_p\si_{p}}{T_{p,t}} \geq \frac{w_p\si_{p}}{T_{p,n}}.
\end{equation}
Combining Equations~\ref{e:meca2}--\ref{e:meca.hb2-2}, we obtain on $\xi$
\begin{equation*}
\frac{w_p\si_{p}}{T_{p,n}}(T_{q,n}-1) \leq w_q\Bigg(\si_{q} + 2\sqrt{2A}\frac{1}{w_q^{1/3} n^{1/3}} \Bigg).
\end{equation*}
Summing over all $q$ such that the previous Equation is satisfied, i.e. such that $T_{q,n} > \lfloor w_q^{2/3}n^{2/3} \rfloor$, on both sides, we obtain on $\xi$
\begin{align*}
\frac{w_p\si_{p}}{T_{p,n}}\sum_{q|T_{q,n} > \lfloor Aw_q^{2/3}n^{2/3} \rfloor}(T_{q,n}-1) \leq \sum_{q|T_{q,n} > \lfloor w_q^{2/3}n^{2/3} \rfloor}w_q \Bigg(\si_{q} +2\sqrt{2A}\frac{1}{w_q^{1/3} n^{1/3}}  \Bigg).
\end{align*}

This implies
\begin{equation}\label{eq:main.lb}
\frac{w_p\si_{p}}{T_{p,n}}(n - \sum_q Aw_q^{2/3}n^{2/3}) \leq \sum_{q=1}^K w_q\Bigg(\si_{q} + 2\sqrt{2A}\frac{1}{w_q^{1/3} n^{1/3}}  \Bigg).
\end{equation}

\noindent
{\bf Step~2.~Lower bound.} Equation~\ref{eq:main.lb} implies
\begin{align*}
\frac{w_p\si_{p}}{T_{p,n}}(n - A\sum_q w_q^{2/3}n^{2/3})
&\leq \Sigma_{\N_n} + \frac{2\sqrt{2A}\sum_q w_q^{2/3}}{n^{1/3}},
\end{align*}
on $\xi$, since $T_{q,n} -1 \geq \frac{T_{q,n}}{2}$ (as $T_{q,n}\geq 2$).
Finally, if $n \geq 2A\sum_q w_q^{2/3}n^{2/3}$, we obtain on $\xi$ the following bound
\begin{equation}\label{eq:almost-lower-bound-2}
\frac{w_p\si_{p}}{T_{p,n}} \leq \frac{\Sigma_{\N_n}}{n} + \big(4\sqrt{2A} + \Sigma_{\N_n} A \big) \frac{\sum_{q\in \N_n} w_q^{2/3}}{n^{4/3}}.
\end{equation}

\noindent
{\bf Step~2bis.~Lower bound on the number of pulls.} By using Equation~\ref{eq:almost-lower-bound-2} and the fact that $\frac{1}{1+x} \geq 1 - x$ one gets

\begin{equation*}
T_{p,n} \geq \lambda_{p,\Sigma_{\N_n}} \Big(n - \frac{\big(4\sqrt{2A} + \Sigma_{\N_n} A \big)}{\Sigma_{\N_n}} \big(\sum_{q\in \N_n} w_q^{2/3}\big) n^{2/3} \Big) \geq \lambda_{p,\Sigma_{\N_n}} \Big(n - B\big(\sum_{q\in \N_n} w_q^{1/3}\big)n^{2/3}\Big),
\end{equation*}
where $B =  \frac{\big(4\sqrt{2A} + \Sigma_{\N_n} A \big)}{\Sigma_{\N_n}}$.

% %
% 
% \noindent
% {\bf Step~3.~Regret.} By summing and using Equation~\ref{eq:almost-lower-bound} which holds for all $p$, we obtain
% 
% \begin{equation}\label{eq:regret1}
% L_n = \sum_p \frac{w_p^2\si_{p}^2}{T_{p,n}} \leq \frac{\Sigma_{\N_n}^2}{n} + \big(4\sqrt{2A} + \Sigma_{\N_n} A \big)\Sigma_{\N_n} \frac{\sum_{q\in \N_n} w_q^{2/3}}{n^{4/3}}.
% \end{equation}

This concludes the proof.
\end{proof}

\section{Proof of Theorem~\ref{th:algo.2}}\label{app:MC-ULCB}

% We consider the same event $\xi$ than in Lemma~\ref{lem:xi} and will prove all results on that event.
% 
% Note that by definition for $j \in \{2i, 2i+1\}$
% \begin{align}
% r_{[h+1,j]} &= \Big( w_{h+1} \hsi_{[h+1,j],t} + 2 \sqrt{A}\frac{w_{h+1}^{2/3}}{n^{1/3}} \Big) \ind{w_{h+1} \hsi_{[h+1,j^-],t} - w_{h+1} \hsi_{[h+1,j],t} \geq \sqrt{2A}\frac{w_{h+1}^{2/3}}{n^{1/3}}} \nonumber\\
% &+ \Big( w_{h+1} \hsi_{[h+1,j],t} - \sqrt{A}\frac{w_{h+1}^{2/3}}{n^{1/3}} \Big) \ind{w_{h+1} \hsi_{[h+1,j^-],t} - w_{h+1} \hsi_{[h+1,j],t} \leq - \sqrt{2A}\frac{w_{h+1}^{2/3}}{n^{1/3}}} \nonumber\\
% &+ \min\Big( w_{h+1}  \min\big( \hsi_{[h+1,j],t},\hsi_{[h+1,j^-],t}\big) + 2 \sqrt{A}\frac{w_{h+1}^{2/3}}{n^{1/3}} , \frac{r_{[h,i]}}{2} \Big) \nonumber\\
% &\times \ind{|w_{h+1} \hsi_{[h+1,j^-],t} - w_{h+1} \hsi_{[h+1,j],t}| \leq  \sqrt{2A}\frac{w_{h+1}^{2/3}}{n^{1/3}}}, \nonumber
% \end{align}
% where $j^-$ is the complementary of $j$ in $\{2i, 2i+1\}$. Note that the three indicators used in the definition of $r$ form a partition of the space.

%\subsection{Study of the first Exploration Phase of the algorithm}

\subsection{Some preliminary bounds}

Let $c = (8\tilde \Sigma +1)\sqrt{A}$. Note that $c \geq 1$.

Let $[h,i]$ be a stratum that is explored during the Exploration Phase, and split in its to children.

This implies that $w_{h} \hsi_{[h,i]}  \geq 6 H c\sqrt{A}\frac{w_{h}^{2/3}}{n^{1/3}}$. By definition, for $j \in \{2i, 2i+1\}$
\begin{align}
r_{[h+1,j]} &= \Big( \frac{w_{h+1} \hsi_{[h+1,j]} + c\sqrt{A}\frac{w_{h+1}^{2/3}}{n^{1/3}}}{w_{h} \tilde \si_{[h,i]}} \Big) r_{[h,i]} \ind{w_{h+1} \hsi_{[h+1,j^-]} - w_{h+1} \hsi_{[h+1,j]} \geq 2c \sqrt{A}\frac{w_{h+1}^{2/3}}{n^{1/3}}} \nonumber\\
&+ \Big( \frac{w_{h+1} \hsi_{[h+1,j]} - c\sqrt{A}\frac{w_{h+1}^{2/3}}{n^{1/3}}}{w_{h} \tilde \si_{[h,i]}} \Big) r_{[h,i]} \ind{w_{h+1} \hsi_{[h+1,j^-]} - w_{h+1} \hsi_{[h+1,j]} \leq - 2c\sqrt{A}\frac{w_{h+1}^{2/3}}{n^{1/3}}} \nonumber\\
&+ \min\Big( \frac{w_{h+1}  \min\big( \hsi_{[h+1,j]},\hsi_{[h+1,j^-]}\big) + c\sqrt{A}\frac{w_{h+1}^{2/3}}{n^{1/3}}}{w_{h} \tilde  \si_{[h,i]}} , \frac{1}{2} \Big)r_{[h,i]} \nonumber\\
&\times \ind{|w_{h+1} \hsi_{[h+1,j^-]} - w_{h+1} \hsi_{[h+1,j]}| \leq  2c\sqrt{A}\frac{w_{h+1}^{2/3}}{n^{1/3}}}, \nonumber
\end{align}
where $j^-$ is the complementary of $j$ in $\{2i, 2i+1\}$. Note that the three indicators used in the definition of $r$ form a partition of the domain.

\begin{lemma}\label{lem:additiv}
If on $\xi$ a node $[h,i]$ has two children $[h+1,2i]$ and $[h+1,2i+1]$ that have been explored by the algorithm, then $r_{[h+1,2i]} + r_{[h+1,2i+1]} \leq r_{[h,i]}$.  
\end{lemma}
\begin{proof}
Note first that $w_{h+1} \hsi_{[h+1,j^-]} + w_{h+1} \hsi_{[h+1,j]} \leq w_{h} \tilde \si_{[h,i]}$ (by definition of $\hsi$ and $\tilde  \si$, and also because of the properties of the empirical variance).

The result follows from the definition of $r$ as for $j  \in \{2i,2i+1\}$, $\Big( \frac{w_{h+1} \hsi_{[h+1,j]} + c\sqrt{A}\frac{w_{h+1}^{2/3}}{n^{1/3}}}{w_{h} \tilde \si_{[h,i]}} \Big) \Big( \frac{w_{h+1} \hsi_{[h+1,j^-]} - c\sqrt{A}\frac{w_{h+1}^{2/3}}{n^{1/3}}}{w_{h} \tilde  \si_{[h,i],t}} \Big)  \leq 1$.
\end{proof}

\begin{lemma}\label{lem:r1}
For any stratum $\X_{[h,i]}$, if $r_{[h,i]}$ of depth smaller than $H$ is defined then on $\xi$
\begin{align*}
 \frac{(2H-h)}{2H}\Big(w_{[h,i]} \hsi_{[h,i]} -  c\sqrt{A} \frac{w_{[h,i]}^{2/3}}{n^{1/3}}\Big) \leq r_{[h,i]} \leq \frac{(H+2h)}{H} \Big(w_{[h,i]} \hsi_{[h,i]} +  c\sqrt{A} \frac{w_{[h,i]}^{2/3}}{n^{1/3}}\Big).
\end{align*}
\end{lemma}
\begin{proof}
The proof is done by induction. Note first that $r_{[0,0]}= w_{[0,0]} \hsi_{[0,0]} + c \sqrt{A} \frac{w_{[0,0]}^{2/3}}{n^{1/3}}$. The result is thus satisfied for node $[0,0]$.

Assume that the property of Lemma~\ref{lem:r1} is satisfied for a given $[h,i]$ on $\xi$. 

Assume that the children of this node are opened. This implies that $w_{h} \hsi_{[h,i]}  \geq 6 H c \sqrt{A}\frac{w_{h}^{2/3}}{n^{1/3}}$, i.e.
\begin{align}\label{eq:2H}
\frac{1}{2H}  \geq 3 c \frac{\sqrt{A}\frac{w_{h}^{2/3}}{n^{1/3}}}{w_{h} \hsi_{[h,i]}}.
\end{align}

Let $j \in \{2i, 2i+1\}$. Note first that $w_{h+1} \hsi_{[h+1,j^-]} + w_{h+1} \hsi_{[h+1,j]} \leq w_{h} \tilde \si_{[h,i]}$ (by definition of $\hsi$ and $\tilde  \si$, and also because of the properties of the empirical variance), and that on $\xi$, $|w_{h} \tilde  \si_{[h,i]} - w_{h} \hsi_{[h,i]}|\leq 2  \sqrt{A} \frac{w_{[h,i]}^{2/3}}{n^{1/3}}$ as a node is open only if there are enough samples in it, i.e.~if there are more than $\lfloor A w_{[h,i]}^{2/3}n^{2/3} \rfloor$ samples. This together with Equation~\ref{eq:2H} implies that 
\begin{align}\label{eq:moins}
\frac{w_{[h,i]} \hsi_{[h,i]} -  c\sqrt{A} \frac{w_{[h,i]}^{2/3}}{n^{1/3}}}{w_{[h,i]} \tilde \si_{[h,i]}} \geq \frac{w_{[h,i]} \tilde  \si_{[h,i],t} -  3 c\sqrt{A} \frac{w_{[h,i]}^{2/3}}{n^{1/3}}}{w_{[h,i]} \tilde  \si_{[h,i]}} \geq 1 - \frac{1}{2H}.,
\end{align}
as $c \geq 1$. In the same way
%In the same way, 
\begin{align}\label{eq:plus}
 \frac{w_{[h,i]} \hsi_{[h,i]} +  c\sqrt{A} \frac{w_{[h,i]}^{2/3}}{n^{1/3}}}{w_{[h,i]} \tilde  \si_{[h,i]}} \leq 1 + \frac{1}{2H}.
\end{align}

%Assume that $|w_{h+1} \hsi_{[h+1,j^-],t} - w_{h+1} \hsi_{[h+1,j],t}| \geq 2\sqrt{A}\frac{w_{h+1}^{2/3}}{n^{1/3}}$.
By Equation~\ref{eq:moins}
\begin{align}
\Big( \frac{w_{h+1} \hsi_{[h+1,j]} - c\sqrt{A}\frac{w_{h+1}^{2/3}}{n^{1/3}}}{w_{h} \tilde \si_{[h,i]}} \Big) r_{[h,i]} &\geq \Big( w_{h+1} \hsi_{[h+1,j]} - c\sqrt{A}\frac{w_{h+1}^{2/3}}{n^{1/3}} \Big) \big(\frac{2H - h}{2H}\big) \big( 1 - \frac{1}{2H}\big)\nonumber \\
&\geq \Big( w_{h+1} \hsi_{[h+1,j]} - c\sqrt{A}\frac{w_{h+1}^{2/3}}{n^{1/3}} \Big) \big(\frac{2H - (h+1)}{2H}\big). \label{eq:moins1}
\end{align}
In the same way, by Equation~\ref{eq:plus}
\begin{align}
\Big( \frac{w_{h+1} \hsi_{[h+1,j]} + c\sqrt{A}\frac{w_{h+1}^{2/3}}{n^{1/3}}}{w_{h} \tilde \si_{[h,i]}} \Big) r_{[h,i]} &\leq \Big( w_{h+1} \hsi_{[h+1,j]} + c\sqrt{A}\frac{w_{h+1}^{2/3}}{n^{1/3}} \Big) \big(\frac{H + 2h}{H}\big) \big( 1 + \frac{1}{2H}\big)\nonumber \\
&\leq \Big( w_{h+1} \hsi_{[h+1,j]} + c\sqrt{A}\frac{w_{h+1}^{2/3}}{n^{1/3}} \Big) \big(1 + \frac{2h}{H} + \frac{1}{2H} + \frac{h}{H^2} \big)\nonumber \\
&\leq \Big( w_{h+1} \hsi_{[h+1,j]} + c\sqrt{A}\frac{w_{h+1}^{2/3}}{n^{1/3}} \Big) \big(1 + \frac{2h}{H} + \frac{3}{2H} \big)\nonumber \\
&\leq \Big( w_{h+1} \hsi_{[h+1,j]} + c\sqrt{A}\frac{w_{h+1}^{2/3}}{n^{1/3}} \Big) \big(\frac{H + 2(h+1)}{H}\big), \label{eq:plus1}
\end{align}
as $h \leq H$.

Assume that $|w_{h+1} \hsi_{[h+1,j]} - w_{h+1} \hsi_{[h+1,j^-]}| \leq 2c\sqrt{A}\frac{w_{h+1}^{2/3}}{n^{1/3}}$. Then $\frac{w_{h+1} \hsi_{[h+1,j]} - c\sqrt{A}\frac{w_{h+1}^{2/3}}{n^{1/3}}}{w_{h+1} \tilde  \si_{[h,i]}} \leq \frac{1}{2}$. It implies that, by Equation~\ref{eq:moins1}
\begin{align}
\frac{r_{[h,i]}}{2} &\geq \Big( \frac{w_{h+1} \hsi_{[h+1,j]} - c\sqrt{A}\frac{w_{h+1}^{2/3}}{n^{1/3}}}{w_{h} \tilde \si_{[h,i]}} \Big) r_{[h,i]} \nonumber \\
&\geq \Big( w_{h+1} \hsi_{[h+1,j]} - c\sqrt{A}\frac{w_{h+1}^{2/3}}{n^{1/3}} \Big) \big(\frac{2H - (h+1)}{2H}\big). \label{eq:moins2}
\end{align}
%As this is true for $j$ such that $w_{h+1} \hsi_{[h+1,j],t} \geq w_{h+1} \hsi_{[h+1,j^-],t}$, it is a fortiori true for the sub-stratum of smallest variance.

Assume that $|w_{h+1} \hsi_{[h+1,j]} - w_{h+1} \hsi_{[h+1,j^-]}| \geq - 2c\sqrt{A}\frac{w_{h+1}^{2/3}}{n^{1/3}}$. Then $\frac{w_{h+1} \hsi_{[h+1,j]} + c\sqrt{A}\frac{w_{h+1}^{2/3}}{n^{1/3}}}{w_{h+1} \tilde \si_{[h,i]}} \geq \frac{1}{2}$. It implies that, by by Equation~\ref{eq:plus1}
\begin{align}
\frac{r_{[h,i]}}{2} &\leq \Big( \frac{w_{h+1} \hsi_{[h+1,j]} + c\sqrt{A}\frac{w_{h+1}^{2/3}}{n^{1/3}}}{w_{h} \tilde \si_{[h,i]}} \Big) r_{[h,i]} \nonumber \\
&\leq \Big( w_{h+1} \hsi_{[h+1,j]} + c\sqrt{A}\frac{w_{h+1}^{2/3}}{n^{1/3}} \Big) \big(\frac{H + 2(h+1)}{H}\big). \label{eq:plus2}
\end{align}
%As this is true for $j$ such that $w_{h+1} \hsi_{[h+1,j],t} \leq w_{h+1} \hsi_{[h+1,j^-],t}$, it is a fortiori true for the sub-stratum of biggest variance.

From Equations~\ref{eq:moins1} and~\ref{eq:moins2}, from the definition of $r$, and from the fact that $\Big( \frac{w_{h+1} \hsi_{[h+1,j]} - c\sqrt{A}\frac{w_{h+1}^{2/3}}{n^{1/3}}}{w_{h} \tilde \si_{[h,i]}} \Big) r_{[h,i]} \leq \Big( \frac{w_{h+1} \hsi_{[h+1,j]} + c\sqrt{A}\frac{w_{h+1}^{2/3}}{n^{1/3}}}{w_{h} \tilde \si_{[h,i]}} \Big) r_{[h,i]}$, we deduce that
\begin{align*}
r_{[h+1,j]}
&\geq \Big( w_{h+1} \hsi_{[h+1,j]} - c\sqrt{A}\frac{w_{h+1}^{2/3}}{n^{1/3}} \Big) \big(\frac{2H - (h+1)}{2H}\big),
\end{align*}
and finish the induction for the left-hand-side on $\xi$.

In the same way, by combining Equations~\ref{eq:plus1} and~\ref{eq:plus2}, we finish the induction for the right-hand-side on $\xi$.

% 
% \begin{align*}
% \Big(w_{h+1} \hsi_{[h+1,j],t} - \sqrt{A}\frac{w_{h+1}^{2/3}}{n^{1/3}} \leq \Big( \frac{w_{h+1} \hsi_{[h+1,j],t} - \sqrt{A}\frac{w_{h+1}^{2/3}}{n^{1/3}}}{w_{h} \hsi_{[h,i],t}} \Big) r_{[h,i]} \leq \Big( \frac{w_{h+1} \hsi_{[h+1,j],t} + \sqrt{A}\frac{w_{h+1}^{2/3}}{n^{1/3}}}{w_{h} \hsi_{[h,i],t}} \Big) r_{[h,i]} 
% \end{align*}
% 
% 
% By definition, for $j \in \{2i, 2i+1\}$
% \begin{align}
% r_{[h+1,j]} &= \Big( \frac{w_{h+1} \hsi_{[h+1,j],t} + \sqrt{A}\frac{w_{h+1}^{2/3}}{n^{1/3}}}{w_{h} \hsi_{[h,i],t}} \Big) r_{[h,i]} \ind{w_{h+1} \hsi_{[h+1,j^-],t} - w_{h+1} \hsi_{[h+1,j],t} \geq 2\sqrt{A}\frac{w_{h+1}^{2/3}}{n^{1/3}}} \nonumber\\
% &+ \Big( \frac{w_{h+1} \hsi_{[h+1,j],t} - \sqrt{A}\frac{w_{h+1}^{2/3}}{n^{1/3}}}{w_{h} \hsi_{[h,i],t}} \Big) r_{[h,i]} \ind{w_{h+1} \hsi_{[h+1,j^-],t} - w_{h+1} \hsi_{[h+1,j],t} \leq - 2\sqrt{A}\frac{w_{h+1}^{2/3}}{n^{1/3}}} \nonumber\\
% &+ \min\Big( \frac{w_{h+1}  \min\big( \hsi_{[h+1,j],t},\hsi_{[h+1,j^-],t}\big) + \sqrt{A}\frac{w_{h+1}^{2/3}}{n^{1/3}}}{w_{h} \hsi_{[h,i],t}} , \frac{1}{2} \Big)r_{[h,i]} \nonumber\\
% &\times \ind{|w_{h+1} \hsi_{[h+1,j^-],t} - w_{h+1} \hsi_{[h+1,j],t}| \leq  2\sqrt{A}\frac{w_{h+1}^{2/3}}{n^{1/3}}}, \nonumber
% \end{align}
% where $j^-$ is the complementary of $j$ in $\{2i, 2i+1\}$. Note that the three indicators used in the definition of $r$ form a partition of the space.
% 
% 
% 

\end{proof}

\begin{corollary}\label{lem:r}
 For any stratum $\X_{[h,i]}$, if $r_{[h,i]}$ is defined then on $\xi$
\begin{align*}
 \frac{(2H-h)}{2H}\Big(w_{[h,i]} \si_{[h,i]} -  2c\sqrt{A} \frac{w_{[h,i]}^{2/3}}{n^{1/3}}\Big) \leq r_{[h,i]} \leq \frac{(H+2h)}{H} \Big(w_{[h,i]} \si_{[h,i]} +  2c\sqrt{A} \frac{w_{[h,i]}^{2/3}}{n^{1/3}}\Big).
\end{align*}
\end{corollary}
\begin{proof}
 This is straightforward from Lemma~\ref{lem:r1}, by the definition of $\xi$ and as $c \geq 1$.
\end{proof}

\begin{lemma}\label{lem:rbig}
For any stratum $\X_{[h,i]}$, if $r_{[h,i]}$ is defined then on $\xi$
\begin{align*}
r_{[h,i]} \times \Big(\frac{n}{4 \tilde \Sigma}\Big) > A w_h^{2/3} n^{2/3}.
\end{align*}
\end{lemma}
\begin{proof}
%The proof is done by induction. Note first that $r_{[0,0]} = \tilde \Sigma  = w_{[0,0]} \hsi_{[0,0]} + c \sqrt{A} \frac{w_{[0,0]}^{2/3}}{n^{1/3}} \geq  c \sqrt{A} \frac{w_{[0,0]}^{2/3}}{n^{1/3}}$. The result is thus satisfied for node $[0,0]$.

Let $[h,i]$ be a node. 

Assume that the children of this node are explored at time $t$. This implies that $w_{h} \hsi_{[h,i]}  \geq 6 H c \sqrt{A}\frac{w_{h}^{2/3}}{n^{1/3}}$, and then by Lemma~\ref{lem:r1}, on $\xi$, (as $\frac{2H - h}{2H} \geq \frac{1}{2}$).
\begin{align*}
r_{[h,i]} &\geq \frac{1}{2} \Big( w_{h} \hsi_{[h,i]} - c \sqrt{A}\frac{w_{h}^{2/3}}{n^{1/3}} \Big)\\
&\geq \frac{1}{2} \Big( 6 H c \sqrt{A}\frac{w_{h}^{2/3}}{n^{1/3}} - c \sqrt{A}\frac{w_{h}^{2/3}}{n^{1/3}} \Big)\\
&\geq \frac{5}{2}  c \sqrt{A}\frac{w_{h}^{2/3}}{n^{1/3}},
\end{align*}
as $H \geq 2$. This implies as $c > 8 \tilde \Sigma \sqrt{A}$ that
\begin{align}\label{eq:minsamples}
 \frac{r_{[h,i]}}{2} \Big(\frac{n}{4 \tilde \Sigma} \Big) > A w_{h+1}^{2/3} n^{2/3}.
\end{align}

%Assume that $|w_{h+1} \hsi_{[h+1,j^-],t} - w_{h+1} \hsi_{[h+1,j],t}| \geq 2\sqrt{A}\frac{w_{h+1}^{2/3}}{n^{1/3}}$.
By Equation~\ref{eq:moins} (as $\frac{2H - h}{2H} \geq \frac{1}{2}$)
\begin{align}
\Big( \frac{w_{h+1} \hsi_{[h+1,j]} + c\sqrt{A}\frac{w_{h+1}^{2/3}}{n^{1/3}}}{w_{h} \tilde \si_{[h,i]}} \Big) r_{[h,i]} &\geq  \frac{1}{2}\Big( w_{h+1} \hsi_{[h+1,j]} + c\sqrt{A}\frac{w_{h+1}^{2/3}}{n^{1/3}} \Big) \nonumber\\
&\geq \frac{1}{2} c\sqrt{A}\frac{w_{h+1}^{2/3}}{n^{1/3}}. \nonumber
\end{align}
This implies as $c > 8 \tilde \Sigma \sqrt{A}$ that
\begin{align}\label{eq:minsamples2}
\Big( \frac{w_{h+1} \hsi_{[h+1,j]} + c\sqrt{A}\frac{w_{h+1}^{2/3}}{n^{1/3}}}{w_{h} \tilde \si_{[h,i]}} \Big) r_{[h,i]}  \Big(\frac{n}{4 \tilde \Sigma} \Big) > A w_{h+1}^{2/3} n^{2/3} 
\end{align}

Let $j^* = \arg\min_j r_{[h+1,j]}$. For $j = \{2i,2i+1\}$, we know that from the definition of $r$, $r_{[h+1,j]} \geq \min\Big[\Big( \frac{w_{h+1} \hsi_{[h+1,j^*]} + c\sqrt{A}\frac{w_{h+1}^{2/3}}{n^{1/3}}}{w_{h} \tilde \si_{[h,i]}} \Big) r_{[h,i]}, \frac{r_{[h,i]}}{2}\Big]$. From that and Equations~\ref{eq:minsamples} and~\ref{eq:minsamples2} we deduce the  Lemma.

\end{proof}

\subsection{Study of the Exploration Phase}

\begin{lemma}\label{lem:finition}
 On $\xi$, the Exploration phase ends at $T <n$ and all the nodes $x$ of partition $\N^e_n$ are such that $\frac{r_x}{T_{x,T}+1} \leq \frac{4\tilde{\Sigma}}{n}$ and $\frac{r_x}{T_{x,T}} > \frac{4\tilde{\Sigma}}{n}$.
\end{lemma}
\begin{proof}
Let $T$ be the time at which the exploration phase ends (if it does not end, write $T = n$). %In order to prove that the Exploration Phase ends before all the budget is used, it is sufficient to prove that after having pulled all nodes such that $r_x \leq \frac{4\tilde{\Sigma}}{n}$, one has $\sum_{x \in \N^e_n} T_{x,t} < n$.

One needs to pull a node in $\N^e_n$ at a time $t' < T$ if and only if
\begin{align*}
 \frac{r_x}{T_{x,t'}+1} > \frac{4\tilde{\Sigma}}{n}.
\end{align*}
We thus know that the last time stratum $\X_x$ is sampled during the Exploration Phase (and thus at the end of the Exploration Phase)
\begin{align*}
 \frac{r_x}{T_{x,T}} \geq \frac{4\tilde{\Sigma}}{n}.
\end{align*}

If stratum $\X_x$ is not sampled during the Exploration Phase after having been opened, then
\begin{align*}
 T_{x,T} = \lfloor A w_x^{2/3}n^{2/3} \rfloor.
\end{align*}
Note that by Lemma~\ref{lem:rbig}, on $\xi$ $r_x \frac{n}{4 \tilde \Sigma}> A w_x^{2/3} n^{2/3}$. From that we deduce that
\begin{align*}
 \frac{r_x}{T_{x,T}} > \frac{4\tilde{\Sigma}}{n},
\end{align*}
and from that together with the fact that we only sample a node at time $t<T$ if  $\frac{r_x}{T_{x,t}} > \frac{4\tilde{\Sigma}}{n}$, we deduce the second part of the Lemma, i.e.~that on $\xi$, $\forall x \in \N^e_n,  \frac{r_x}{T_{x,T}} > \frac{4\tilde{\Sigma}}{n}$.

%Note also that the maximal depth of the partition is $H = \lfloor \frac{\log\big((3 f_{\max})^3 n\big)}{\log(2)} \rfloor +1$. There are thus less than $(3 f_{\max})^3 n$ strata in partition $\N^e_n$.

Note now that $\sum_{x \in \N_n^e} r_x \leq r_{[0,0]} = \tilde{\Sigma}$: it is straightforward by Lemma~\ref{lem:additiv}. This directly leads to:
\begin{align*}
 \tilde{\Sigma} \geq \sum_{x \in \N_n^e} r_x \geq \frac{4\tilde{\Sigma}}{n} \sum_{x \in \N_n^e}T_{x,T}.
\end{align*}

This directly implies that $\sum_{x \in \N_n^e}T_{x,T} \leq \frac{n}{4} < n$, which leads to the desired result, i.e.~that the Exploration Phase ends before all the budget has been used. This implies that on $\xi$, $\forall x \in \N^e_n,  \frac{r_x}{T_{x,T}+1} \leq \frac{4\tilde{\Sigma}}{n}$.

\end{proof}

\begin{lemma}\label{lem:open}
 
Let $x$ be a node such that $w_x \si_x \geq 14 H c \sqrt{A}\frac{w_{x}^{2/3}}{n^{1/3}}$ and also such that, for all its parents, $w_y \si_y \geq 14 H c \sqrt{A}\frac{w_{y}^{2/3}}{n^{1/3}}$. 

Then on $\xi$, at the end $T$ of the Exploration phase phase, node $x$ is open, i.e.~$x \in \T_n^e$, which also implies $T_{x,T} \geq A w_x^{2/3}n^{2/3} (\geq 2)$.
\end{lemma}
\begin{proof}
The result is proven by induction. Assume that there is a node $x$ that satisfies the Assumptions of Lemma~\ref{lem:open}. Then $w_{[0,0]} \si_{[0,0]} \geq 14 H c \sqrt{A}\frac{w_{[0,0]}^{2/3}}{n^{1/3}}$. Note first that after the Initialization, i.e.~at the time $t = \lfloor A n^{2/3} \rfloor$ when $T_{[0,0],t} = \lfloor A n^{2/3} \rfloor$, i.e.~when the decision of opening or not the node is made, we have on $\xi$ that
\begin{align*}
w_{[0,0]} \hsi_{[0,0]} &\geq w_{[0,0]} \si_{[0,0]} - 2  \sqrt{A}\frac{w_{[0,0]}^{2/3}}{n^{1/3}}\\ 
&\geq 12 H c \sqrt{A}\frac{w_{[0,0]}^{2/3}}{n^{1/3}}\\ 
&\geq 6 H c \sqrt{A}\frac{w_{[0,0]}^{2/3}}{n^{1/3}}.
\end{align*}
The node $[0,0]$ is thus opened on $\xi$ .

Assume now that an ancestor $[h,i]$ of node $x$ is open. By Lemma~\ref{lem:r}, we now that on $\xi$
\begin{align*}
r_{[h,i]} &\geq  \frac{(2H-h)}{2H}\Big(w_{[h,i]} \si_{[h,i],t_{[h,i]}} -  2c\sqrt{A} \frac{w_{[h,i]}^{2/3}}{n^{1/3}}\Big)\\
&\geq \frac{1}{2}\Big(14 H c \sqrt{A}\frac{w_{x}^{2/3}}{n^{1/3}} -  2c\sqrt{A} \frac{w_{[h,i]}^{2/3}}{n^{1/3}}\Big)\\
&\geq 6 H c \sqrt{A}\frac{w_{[h,i]}^{2/3}}{n^{1/3}}.
\end{align*}
By Lemma~\ref{lem:open}, we know that at the end $T$ of the Exploration Phase, with $T <n$ on $ \xi$, we have $\frac{r_{[h,i]}}{T_{[h,i],T}+1} \leq \frac{4 \tilde \Sigma}{n}$. As $c > 8 \tilde \Sigma \sqrt{A}$, we have by using the previous result that $T_{[h,i],T} \geq 6 H A w_{[h,i]}^{2/3}n^{2/3}$. By the definition of $A$ and the fact that $h \leq H$, we know also that $A w_{[h,i]}^{2/3}n^{2/3} \geq 2$, which implies that $T_{[h,i],T}\geq 2$. This, together with the fact that $w_{[h,i]} \hsi_{[h,i],T} \geq 12 H A w_{[h,i]}^{2/3}n^{2/3}$ on $\xi$, implies that node $[h,i]$ is open and split in its too children.

We have thus proved the result of the Lemma by induction.
\end{proof}

\begin{lemma}\label{lem:bosup}
 Let $T$ be the end of the Exploration Phase, and let $x\in \T_n^e$. Then on $\xi$,
\begin{align*}
T_{x,T}  \leq \max\Big( \frac{5w_x \si_x n}{6\tilde \Sigma} , 15 c \sqrt{A} \frac{w_x^{2/3} n^{2/3}}{\tilde \Sigma}\Big).
\end{align*}
\end{lemma}
\begin{proof}
Let $T$ be the end of the exploration phase.

Let $x\in \T^e_n$. Let $\N$ be the subset of the partition $\N^e_n$ that covers $x$. Let $y \in \N$. By Lemma~\ref{lem:finition} we have on $\xi$
\begin{align*}
 \frac{r_y}{T_{y,T}} > 4\frac{\tilde{\Sigma}}{n},
\end{align*}
which leads directly to
\begin{align*}
T_{y,T} < \frac{r_y n}{4\tilde{\Sigma}}.
\end{align*}
Note that by Lemma~\ref{lem:additiv} one has $\sum_{y \in \N} r_y \leq r_x$. One thus has 
\begin{align}\label{eq:borneT}
T_{x,T} = \sum_{y \in \N} T_{y,T} \leq \sum_{y \in \N} \frac{r_y n}{4\tilde{\Sigma}} \leq \frac{r_x n}{4\tilde{\Sigma}}.
\end{align}

Note now that by Corollary~\ref{lem:r}, we have on $\xi$ $r_x \leq 3\Big( w_x \si_x + 2 c \sqrt{A} \frac{w_x^{2/3}}{n^{1/3}}\Big)$. From that and Equation~\ref{eq:borneT}, we deduce that on $\xi$
\begin{align*}
T_{x,T}  &\leq 3\Big( w_x \si_x + 2 c \sqrt{A} \frac{w_x^{2/3}}{n^{1/3}}\Big) \frac{n}{4\tilde{\Sigma}} \nonumber \\
&\leq \max\Big( \frac{5w_x \si_x n}{6\tilde \Sigma} , 15 c \sqrt{A} \frac{w_x^{2/3} n^{2/3}}{\tilde \Sigma}\Big).
\end{align*}
% Note now that a node $x$ only belongs to $\T__n^e$ if $w_x\hsi_x \geq 6Hc\sqrt{A}\frac{w_x^{2/3}}{n^{1/3}}$, which could happen on $\xi$ only if
% \begin{align*}
%  w_x\si_x \geq 4Hc\sqrt{A}\frac{w_x^{2/3}}{n^{1/3}}.
% \end{align*}
% Note that by the definition of $H$
This concludes the proof.

% and as $w_x \si_x \geq C \frac{w_x^{2/3}}{n^{1/3}}$ and $C \geq 4\sqrt{A}$, one has $r_x \leq 2 w_x \si_x$. Note also that $\tilde{\Sigma} \geq \Sigma$. Both these results lead to
% \begin{align*}
% T^e_x  \leq \frac{w_x\si_x n}{2\Sigma}.
% \end{align*}

% 
% Let $x$ be a node in such that $w_x \si_x < C \frac{w_x^{2/3}}{n^{1/3}}$. Let $\N$ be the subset of the partition $\N^e_T$ that covers $x$. Let $y \in \N$. By the definition of $T^e_y$, one has
% \begin{align*}
%  \frac{r_y}{T_y^e} \geq 4\frac{\tilde{\Sigma}}{n},
% \end{align*}
% which leads directly to
% \begin{align*}
% T^e_y \leq \frac{r_y n}{4\tilde{\Sigma}}.
% \end{align*}
% 
% 
% Note that by construction of the $r_y$, one has $\sum_{y \in \N} r_y \leq r_x$. One thus has 
% \begin{align*}
% T_x = \sum_{y \in \N} T^e_y \leq \sum_{y \in \N} \frac{r_y n}{4\tilde{\Sigma}} \leq \frac{r_x n}{4\tilde{\Sigma}}.
% \end{align*}
% 
% Note now that $r_x \leq  w_x \si_x + 2 \sqrt{A} \frac{w_x^{2/3}}{n^{1/3}}$ and as $w_x \si_x < C \frac{w_x^{2/3}}{n^{1/3}}$, one has $r_x \leq (C + 2\sqrt{A})\frac{w_x^{2/3}}{n^{1/3}}$. Note also that $\tilde{\Sigma} \geq \Sigma$. Both these results lead to
% \begin{align*}
% T^e_x  \leq \frac{(C + 2\sqrt{A})w_x^{2/3}n^{2/3}}{4\Sigma}.
% \end{align*}

\end{proof}

% \subsection{Study of the Filling Phase}
% 
% 
% \begin{lemma}\label{lem:finition2}
%  At the end $T$ of the Filling phase all the nodes $x$ of the tree $\T_n^e(\N_n)$ are such that $\frac{r_x}{T_{x,T}} \leq \frac{4\tilde{\Sigma}}{n}$, and $T <n$.
% \end{lemma}
% 
% \begin{proof}
% Let $x \in \T_n^e(\N_n)$. Let $T_{x,T}$ be the number of points needed in order for node $x$ to be such that $\frac{r_x}{T_{x,T}+1} \leq \frac{4\tilde{\Sigma}}{n}$. Note that by Lemma~\ref{lem:additiv} and Lemma~\ref{lem:finition}, we know that at the end $t$ of the Exploration Phase
% 
% 
% Let $x$ be in $\T_n^e(\N_n)$. Let $T_{x,T}$ be the number of points needed in order for node $x$ to be such that $\frac{r_x}{T_{x,T}+1} \leq \frac{4\tilde{\Sigma}}{n}$.
%  In order to prove the Lemma, it is sufficient to prove that after having pulled all nodes such that $r_x \leq \frac{4\tilde{\Sigma}}{n}$, one has for any partition $\N$ of nodes in $\T_n^e$ that $\sum_{x \in \N} T_{x,T} < n$.
% 
% Note first that one needs to pull a node $x \in \T^e_n$ if and only if
% \begin{align*}
%  \frac{r_x}{T_{x,T}} \geq \frac{4\tilde{\Sigma}}{n}.
% \end{align*}
% 
% 
% Note now that $\sum_{x \in \N} r_x \leq r_{[0,0]} = \tilde{\Sigma}$: it is straightforward from Lemma~\ref{lem:additiv}. This directly leads to:
% \begin{align*}
%  \tilde{\Sigma} \geq \sum_{x \in \N} r_x \geq \frac{4\tilde{\Sigma}}{n} \sum_{x \in \N}T_{x,T}.
% \end{align*}
% 
% This directly implies that $\sum_{x \in \N_n}T_{x,T} \leq \frac{n}{4} < n$, which leads to the desired result.
% 
% \end{proof}
% 
% 

\subsection{Characterization of the $\Sigma_{\N_n}$}

%Note first that if at a moment a node is such that $w_x \si_x \geq \frac{w_x^{2/3}}{n^{1/3}}$, then it means that it is not worth opening it: indeed $C_{\min} \geq 1$ because $A \geq 1$. 

%Let $\mathbb{S}$ be the set of all nodes $x$ such that all their ancestors $y$ are such that $w_y \si_y \geq 14 H c\sqrt{A} \frac{w_x^{2/3}}{n^{1/3}}$. Lemma~\ref{lem:open} states that on $\xi$, $\mathbb{S} \subset \T^e_n$.% Additionally, the same Lemma states that on $\xi$, all nodes in $\T^e_n$ are such that  $T_{x,n} \geq A w_x^{2/3}n^{2/3}$.

The algorithm selects a partition $\N_n$ such that
\begin{align*}
 \N_n \in  \arg\min_{\N \in \T^e_n} \Big( \hat{\Sigma}_{\N} + (C_{\max}' - \sqrt{A}) \sum_{y\in{\N}}\frac{w_y^{2/3}}{n^{1/3}} \Big),
\end{align*}
%with $C_{\max}' = 4 \sqrt{2A} + f_{\max} A + 2\sqrt{A}$ and 
with $C_{\max}' = \max(B,14 H c\sqrt{A}) + 2\sqrt{A}$ and $B =16\sqrt{2A} c  (1 + \frac{1}{\tilde\Sigma})$.

Note that for every partition $\N \in  \T_n^e$, as all the nodes of $ \T_n^e$ are such that  $T_{x,n} \geq A w_x^{2/3}n^{2/3}\geq 2$ by the structure of the algorithm. One thus has on $\xi$, for any $\N$ partition included in $\T^e_n$, that
\begin{align*}
|\hat{\Sigma}_{\N} - \Sigma_{\N}| \leq \sqrt{A} \sum_{y\in{\N}}\frac{w_y^{2/3}}{n^{1/3}},
\end{align*}
because by construction every node of $\T^e_n$ has depth smaller than $H$.

We thus have for the selected partition $\N_n$ that, on $\xi$,
\begin{align}\label{eq:1}
 \Sigma_{\N_n} +  (C_{\max}' - 2\sqrt{A}) \sum_{y\in{\N_n}}\frac{w_y^{2/3}}{n^{1/3}} \leq \min_{\N \in \T^e_n}  \Bigg[\Sigma_{\N} +  C_{\max}' \sum_{y\in{\N}}\frac{w_y^{2/3}}{n^{1/3}} \Bigg].
\end{align}

Let $\mathbb{S}$ be the set of all nodes $x$ such that all their ancestors $y$ are such that $w_y \si_y \geq 14 H c\sqrt{A} \frac{w_x^{2/3}}{n^{1/3}}$. This implies because $\si_y$ is positive, and because $C'_{\max} \geq 14 H c \sqrt{A}$ that
\begin{align}\label{eq:2}
 \min_{\N \in \mathbb{S}}  \Bigg[\Sigma_{\N} +    C_{\max}'  \sum_{y\in{\N}}\frac{w_y^{2/3}}{n^{1/3}} \Bigg]  = \min_{\N}  \Bigg[\Sigma_{\N} +    C_{\max}'  \sum_{y\in{\N}}\frac{w_y^{2/3}}{n^{1/3}} \Bigg],
\end{align}
where $\min_{\N}$ is the minimum over all the partitions in the entire hierarchical partitioning.

Lemma~\ref{lem:open} states that on $\xi$, $\mathbb{S} \subset \T^e_n$. This implies that
\begin{align}\label{eq:3}
 \min_{\N \in \T^e_n}  \Bigg[\Sigma_{\N} +  C_{\max}' \sum_{y\in{\N}}\frac{w_y^{2/3}}{n^{1/3}} \Bigg] \leq \min_{\N \in \mathbb{S}}  \Bigg[\Sigma_{\N} +  C_{\max}' \sum_{y\in{\N}}\frac{w_y^{2/3}}{n^{1/3}} \Bigg].
\end{align}

By combining Equations~\ref{eq:1}, \ref{eq:2} and~\ref{eq:3}, we obtain on $\xi$
\begin{align}\label{eq:sigmasigma}
 \Sigma_{\N_n} +  B \sum_{y\in{\N_n}}\frac{w_y^{2/3}}{n^{1/3}} \leq \min_{\N}  \Bigg[\Sigma_{\N} +  C_{\max}' \sum_{y\in{\N}}\frac{w_y^{2/3}}{n^{1/3}} \Bigg].
\end{align}
since $C_{\max}' - 2\sqrt{A}\geq B$.

% 
% and as $C_{\max}' - 2 \sqrt{A} \geq C_{\min}$, and as also the nodes not in $\mathbb{S}$ are not worth opening, one has on any partition of the hierarchical partitioning that
% \begin{align*}
%  \Sigma_{\N_n} +  C_{\min} \sum_{y\in{\N_n}}\frac{w_y^{2/3}}{n^{1/3}} \leq \min_{\N}\Big[ \Sigma_{\N} + C_{\max}' \sum_{y\in{\N}}\frac{w_y^{2/3}}{n^{1/3}} \Big].
% \end{align*}
% 
% 
% Note that if this holds for any $\N \in \mathbb{S}$, this also holds for any $\N$, because the nodes that are not in $\mathbb S$ are such that $w_x \si_x \geq \frac{w_x^{2/3}}{n^{1/3}}$.

\subsection{Study of the Exploitation phase}

\begin{lemma}\label{lem:explo}
 At the end of the Exploitation phase (end of the algorithm) one has $\forall x \in \N_n$ 
\begin{align*}
 \frac{w_x\si_{x}}{T_{x,n}} 
&\leq \frac{\Sigma{\N_n}}{n} + B \sum_{y \in \N_n} \frac{w_y^{2/3}}{n^{1/3}} ,
\end{align*}
where $B =  16\sqrt{2A} c  (1 + \frac{1}{\tilde\Sigma})$.
\end{lemma}
\begin{proof}

{\bf Step~1.~Lower Bound in each node} Let us first note that by Lemma~\ref{lem:finition}, we know that on $\xi$, at the end $T<n$ of the Exploration Phase, we have $\sum_{x \in \N^e_x} T_{x,T} < \frac{n}{4}$. There is still a budget of at least $\frac{3n}{4}$ pulls left for the Exploitation phase.% Note first that by the definition of $\N_n$ and as $\tilde \Sigma \leq f_{\max}^2 \leq \frac{n}{8}$, we have $\frac{n}{4} > \sum_{x \in \N_n} \frac{2c\sqrt{A}}{\tilde \Sigma}w_x^{2/3} n^{2/3}$.
Note first that as a node $x$ is opened only when there are $\lfloor A w_x^{2/3}n^{2/3}\rfloor$ points in it, so $\forall x \in \N_n, T_{x,T} > \frac{A}{2} w_x^{2/3}n^{2/3}$. %This implies that $\frac{n}{4} > \sum_{x \in \N_n} \frac{A}{2}w_x^{2/3} n^{2/3}$. Note also that by definition of $c$ and $A$, we have $\frac{n}{8} > \sum_{x \in \N_n} \frac{15c\sqrt{A}}{\tilde \Sigma}w_x^{2/3} n^{2/3}$.
%\todo{A revoir}

% Note also that by the definition of the algorithm, there are at least $A w_x^{2/3} n^{2/3}$ points in each stratum in $\N_n$.

% It is thus possible to add up to $\sqrt{A}w_x^{2/3} n^{2/3}$ points in each node $x \in \N_n$ and have still some budget left.

{\bf Step~2.~Properties of the algorithm.} We first recall the definition of $B_{q,t+1}$ used in the MC-UCB algorithm for a node $q \in \N_n$
\begin{equation*}
B_{q,t+1} = \frac{w_q}{T_{q,t}} \Bigg( \hsi_{q} + \sqrt{A}\frac{1}{w_q^{1/3}n^{1/3}}\Bigg).
\end{equation*}

Using the definition of $\xi$ together with the fact that, by construction, at a time $t$ of the Exploration Phase, $T_{q,t} \geq \lfloor A w_q^{2/3} n^{2/3}\rfloor$, it follows that, on $\xi$
\begin{equation}\label{e:bound.B}
 \frac{w_q\si_q}{T_{q,t}} \leq B_{q,t+1} \leq \frac{w_q}{T_{q,t}}\Bigg(\si_{q} + 2\sqrt{A}\frac{1}{w_q^{1/3}n^{1/3}} \Bigg).
\end{equation}
Let $t+1 \geq T+1$ be the time at which an arm $q$ is pulled for the last time, that is $T_{q,t} = T_{q,n}-1$. Note that there is at least one arm such that this happens as $n > T$ by Lemma~\ref{lem:finition}. Since at $t+1$ arm $q$ is chosen, then for any other arm $p$, we have
\begin{equation}\label{e:meca2}
B_{p,t+1} \leq B_{q,t+1}\;.
\end{equation}
From Equation~\ref{e:bound.B} and $T_{q,t} = T_{q,n} -1$, we obtain on $\xi$
\begin{equation}\label{e:meca.lb2}
B_{q,t+1} \leq \frac{w_q}{T_{q,t}}\Bigg(\si_{q} + 2 \sqrt{A}\frac{1}{w_q^{1/3}n^{1/3}}\Bigg) = \frac{w_q}{T_{q,n}-1}\Bigg(\si_{q} + 2\sqrt{2A}\frac{1}{w_q^{1/3}n^{1/3}} \Bigg).
\end{equation}
Furthermore, since $T_{p,t} \leq T_{p,n}$, then on $\xi$
\begin{equation}\label{e:meca.hb2}
B_{p,t+1} \geq \frac{w_p\si_{p}}{T_{p,t}} \geq \frac{w_p\si_{p}}{T_{p,n}}.
\end{equation}
Combining Equations~\ref{e:meca2}--\ref{e:meca.hb2}, we obtain on $\xi$ that if at least one sample is collected from stratum $q$ after the Exploration Phase, then
\begin{equation} \label{eq:dyna}
\frac{w_p\si_{p}}{T_{p,n}}(T_{q,n}-1) \leq w_q\Bigg(\si_{q} + 2\sqrt{2A}\frac{1}{w_q^{1/3}n^{1/3}} \Bigg).
\end{equation}

{\bf Step 3: The Exploration Phase has not deteriorate the performances of the algorithm.}

If $T_{y,n}  > T_{y,T}$, then samples are pulled from $y$ after the Exploration Phase. By summing over these nodes on Equation~\ref{eq:dyna}, we obtain that, on $\xi$, for any $x$,
% 
% $T+1$ is the time at which the Exploration Phase begins. Let us now consider the node (or one of the nodes) $x \in \N_n$ such that $w_x \si_x \geq \frac{w_x^{2/3}}{n^{1/3}}$ and such that $B_{x,T}$ is the smallest among these nodes. Assume that node $x$ is not pulled until the end of the algorithm. This implies for all $y$ such that $T_{y,T} \geq \frac{(C + 2\sqrt{A})w_y^{2/3}n^{2/3}}{4\Sigma}$ that
% \begin{align*}
%  \frac{w_x\si_{x}}{T_{x,T}}(T_{y,n}-1) \leq w_y\Bigg(\si_{y} + 2A \sqrt{\frac{1}{T_{y,n}-1}} \Bigg).
% \end{align*}
% 
% Note that by Lemma~\ref{lem:bosup}, all the nodes such that $w_x \si_x < C \frac{w_x^{2/3}}{n^{1/3}}$ are also such that $T_{x,T} < \frac{(C + 2\sqrt{A})w_x^{2/3}n^{2/3}}{4\Sigma}$. Summing over all $q$ such that $T_{q,T} \geq \frac{(C + 2\sqrt{A})w_q^{2/3}n^{2/3}}{4\Sigma}$ is satisfied, on both sides, we obtain on $\xi$
%
\begin{align}
\frac{w_x\si_{x}}{T_{x,n}}\sum_{y| T_{y,n} > T_{y,T}}(T_{y,n}-1) &\leq \sum_{y|T_{y,n} > T_{y,T}}w_y \Bigg(\si_{y} +2\sqrt{2A}\frac{1}{w_y^{1/3}n^{1/3}} \Bigg)\nonumber \\
&\leq \Sigma^- + \frac{2\sqrt{2A} \sum_{y|T_{y,n} > T_{y,T}} w_y^{2/3}}{n^{1/3}} \nonumber \\
&\leq \Sigma^- + \frac{2\sqrt{2A} \sum_{y \in \N_n} w_y^{2/3}}{n^{1/3}}.\label{eq:equa1}
\end{align}
where $\Sigma^- =  \sum_{y|T_{y,n} > T_{y,T}}w_y \si_{y}$. The passage from line $2$ to line $3$ come from the fact that $T_{y,n}\geq T_{y,T} \geq \frac{A}{2} \frac{w_y^{2/3}}{n^{1/3}}$.

Lemma~\ref{lem:bosup} states that on $\xi$, for all $x \in \N_n \subset \T_n^e$
\begin{align*}
T_{x,T}  &\leq \max\Big( \frac{3}{4}\lambda_{x,\N_n} n, 15 c \sqrt{A} \frac{w_x^{2/3} n^{2/3}}{\tilde \Sigma}\Big).
\end{align*}
Note also that by Step 1, on $\xi$, $\frac{3n}{4} \leq \sum_{y| T_{y,n} > T_{y,T}} T_{y,n}$. %Note also that $\sum_{y| T_{y,n} > T_{y,T}} T_{y,n} =   n - \sum_{y| T_{y,n} = T_{y,T}} \frac{1}{2}\lambda_{y,\N_n} n \geq $. This implies that $\sum_{y| T_{y,n} > T_{y,T}} \lambda_{y,\N_n} n \leq \frac{n}{2}$. From that, we deduce that $\frac{\Sigma^-}{\Sigma} \geq \frac{1}{2}$.
We thus have from these two results that on $\xi$, for any $x\in \N_n$,
\begin{align}
 \frac{w_x\si_{x}}{T_{x,n}}\sum_{y| T_{y,n} > T_{y,T}}(T_{y,n}-1) &\geq \frac{w_x\si_{x}}{T_{x,n}} \max \Big[\Big(n - \sum_{y| T_{y,n} = T_{y,T}} \frac{3}{4}\lambda_{x,\N_n} n - \sum_y  15 c \sqrt{A} \frac{w_y^{2/3} n^{2/3}}{\tilde \Sigma}\Big), \frac{3n}{4}\Big] \nonumber\\
&= \frac{w_x\si_{x}}{T_{x,n}}\max\Big[\Big(n \frac{\Sigma^-}{\Sigma_{\N_n}} + n \frac{(\Sigma_{\N_n}- \Sigma^-)}{4\Sigma_{\N_n}} - \sum_y  15 c \sqrt{A} \frac{w_y^{2/3} n^{2/3}}{\tilde \Sigma}\Big), \frac{3n}{4}\Big]. \label{eq:equa2}
\end{align}
%where the last line comes from the remark in Step 1 that $\frac{n}{4} > \sum_{x \in \N_n} \frac{2c\sqrt{A}}{\tilde \Sigma}w_x^{2/3} n^{2/3}$.

%This last Equation together with the fact that $T_{y,n} \geq A \frac{w_y^{2/3}}{n^{1/3}}$, and also because $\sum_y  2 c \sqrt{A} \frac{w_y^{2/3} n^{2/3}}{\tilde \Sigma} \leq \frac{n}{8}$ (see Step 1) lead to
By combining Equations~\ref{eq:equa1} and Equation~\ref{eq:equa2}, we obtain for every $x \in \N_n$ that on $\xi$
\begin{align*}
\frac{w_x\si_{x}}{T_{x,n}} \leq& \frac{1}{\max\Big[\Big(n \frac{\Sigma^-}{\Sigma_{\N_n}} + n \frac{(\Sigma_{\N_n}- \Sigma^-)}{4\Sigma_{\N_n}} - \sum_y  15 c \sqrt{A} \frac{w_y^{2/3} n^{2/3}}{\tilde \Sigma}\Big), \frac{3n}{4}\Big]}\\
 &\Bigg[\Sigma^- + \frac{2\sqrt{2A} \sum_{y \in \N_n} w_y^{2/3}}{n^{1/3}} \Bigg]\\
\leq&  \frac{\Sigma{\N_n}}{n} + \frac{8\sqrt{2A} \sum_{y \in \N_n} w_y^{2/3}}{n^{4/3}}  + 30 \sum_y  c \sqrt{A} \frac{w_y^{2/3} }{n^{4/3}\tilde \Sigma}\\
\leq& \frac{\Sigma{\N_n}}{n} + \frac{38\sqrt{2A} c \sum_{y \in \N_n} w_y^{2/3}}{n^{4/3}} (1 + \frac{1}{\tilde\Sigma}) ,
\end{align*}
where we use the fact that $n \frac{\Sigma^-}{\Sigma_{\N_n}} + n \frac{(\Sigma_{\N_n}- \Sigma^-)}{4\Sigma_{\N_n}} \geq \frac{n}{4}$ and $\frac{1}{1-x} \leq 1+x$ for $x <1$ for passing from line $1$ to line $2$. We finally have
\begin{align}\label{eq:almost-lower-bound}
 \frac{w_x\si_{x}}{T_{x,n}} 
&\leq \frac{\Sigma{\N_n}}{n} + B \sum_{y \in \N_n} \frac{w_y^{2/3}}{n^{4/3}} ,
\end{align}
where $B =  38\sqrt{2A} c  (1 + \frac{1}{\tilde\Sigma})$.

\noindent
{\bf Step~4.~Lower bound on the number of pulls.} By using Equation~\ref{eq:almost-lower-bound} and the fact that $\frac{1}{1+x} \geq 1 - x$ one gets
\begin{equation*}
T_{p,n} \geq \lambda_{p,\Sigma_{\N_n}} \Big(n - \frac{B}{\Sigma_{\N_n}} \big(\sum_{q\in \N_n} w_q^{2/3}\big) n^{2/3} \Big).
\end{equation*}
%where $B =  \frac{\big(4\sqrt{2A} + \Sigma_{\N_n} A \big)}{\Sigma_{\N_n}}$.
\end{proof}

\begin{lemma}\label{lem:egalite}
 Let $x \in \N_n$. Let $y$ be an open grand-child of $x$, and $y_1$ and $y_2$ be its two children. Then
\begin{align*}
 \frac{r_{y_i}}{T_{y_i,n}} \leq  \frac{r_{y_1} + r_{y_2}}{T_{y,n} - 1},
\end{align*}
where $i \in \{1,2\}$.
\end{lemma}
\begin{proof}

We consider $x\in \N_n$ such that $w_x\hsi_x \geq 6Hc\sqrt{A} \frac{w_x^{2/3}}{n^{1/3}}$: otherwise it has no grand-children.
%Note first that 

By Lemma~\ref{lem:bosup}, we know that for any $y$ grand-child of $x$, we have $\frac{r_y n}{4 \tilde \Sigma} \leq A w_y^{2/3} n^{2/3}$. Note that at the moment of a node's opening, the number of points in the node is smaller than $A w_y^{2/3} n^{2/3}$. As the Exploration stops sampling in a stratum $x$ when $ \frac{r_y}{T_{y,n}+1} \leq  4 \frac{\tilde \Sigma}{n}$, we know that at the end $T$ of the Exploration Phase, we have $\frac{r_y}{T_{y,T}} \geq 4 \frac{\tilde \Sigma}{n}$.

We prove by induction that $ \frac{r_y}{T_{y,n}}  \leq 4 \frac{\tilde \Sigma}{n}$ for any grand-child of $x$, and that for its two children $y_1$ and $y_2$, we have $ \frac{r_{y_i}}{T_{y_i,n}} \leq  \frac{r_{y_1} + r_{y_2}}{T_{y,n} - 1}$.

By Lemma~\ref{lem:r1}, we know that as $w_x\hsi_x \geq 6Hc\sqrt{A} \frac{w_x^{2/3}}{n^{1/3}}$, we have on $\xi$
\begin{align*}
r_{x} &\leq 3\Big(w_{x} \si_{x} +  c\sqrt{A} \frac{w_{[h,i]}^{2/3}}{n^{1/3}}\Big) \leq 3 \Big(\frac{7}{6}w_{x} \si_{x} \Big) \leq \frac{7}{2} w_{x} \si_{x}.
\end{align*}
By combining this result with Lemma~\ref{lem:explo} and also with the definition of $\Sigma_{\N_n}$, we have on $\xi$
\begin{align*}
 \frac{r_x}{T_{x,n}} &\leq \frac{7w_x\si_{x}}{2T_{x,n}} \leq \frac{7}{2}\Big(\frac{\Sigma_{\N_n}}{n} + B \sum_{y \in \N_n} \frac{w_y^{2/3}}{n^{4/3}}\Big)
\leq \frac{7}{2}\Big( \frac{w_{[0,0]}\si_{[0,0]}}{n} + \frac{C_{\max}'}{n^{4/3}}\Big)
\leq \frac{7}{2} \frac{\tilde \Sigma}{n},
\end{align*}
because by definition, $\Sigma_{\N_n}+ B \sum_{y \in \N_n} \frac{w_y^{2/3}}{n^{1/3}} \leq \si_{[0,0]} + \frac{C'_{\max}}{n^{1/3}} $, and also because $\tilde \Sigma \leq \si_{[0,0]} +  \frac{C_{\max}'}{n^{1/3}}$.

Let $x_1$ and $x_2$ be the two children of $x$. Note first that at the end $T$ of the Exploration Phase, by Lemma~\ref{lem:finition}, we have $\frac{r_{x_i}}{T_{x_i,T}} \geq 4 \frac{\tilde \Sigma}{n}$, where $i \in \{1,2\}$. By Lemma~\ref{lem:additiv}, we know that $r_x \geq r_{x_1} +r_{x_2} \geq T_{x,T} 4 \frac{\tilde \Sigma}{n}$. This means that as $\frac{7}{2} < 4$, then then a sample will be pulled again in one of the two nodes $\{x_1, x_2\}$ after the Exploration Phase. Assume without risk of generality that it is node $x_1$ that is pulled.
\begin{align*}
 \frac{r_{x_2}}{T_{x_2,n}} \leq  \frac{r_{x_1}}{T_{x_1,n}-1}.
\end{align*}
Note also that $\frac{r_{x_2}}{T_{x_2,n}} \leq  \frac{r_{x_2}}{T_{x_2,n}}$. By summing, we get that
\begin{align*}
 \frac{r_{x_2}}{T_{x_2,n}}(T_{x_1,n} + T_{x_2,n}-1 ) \leq  r_{x_1} + r_{x_2}.
\end{align*}
We thus have
\begin{align*}
 \frac{r_{x_2}}{T_{x_2,n}} \leq  \frac{r_{x_1} + r_{x_2}}{(T_{x_1,n} + T_{x_2,n} -1)} \leq  \frac{r_{x_1} + r_{x_2}}{T_{x,n}-1} .
\end{align*}
If a sample is also collected from stratum $x_2$, then the same result applies also for $x_1$. Otherwise, it means that $\frac{r_{x_2}}{T_{x_2,n}} = \frac{r_{x_2}}{T_{x_2,T}} \geq 4 \frac{\tilde \Sigma}{n}$, and as one sample is collected in $x_1$, we have $\frac{r_{x_1}}{T_{x_1,n}} \leq 4 \frac{\tilde \Sigma}{n}$, so we have in any case
\begin{align*}
 \frac{r_{x_1}}{T_{x_1,n}} \leq  \frac{r_{x_1} + r_{x_2}}{T_{x,n} -1} .
\end{align*}

The recursion continues in the same way for any child $y$ of $x$ such that $w_y \hsi_y \geq 6Hc\sqrt{A} \frac{w_y^{2/3}}{n^{1/3}}$ (otherwise it has no children). Indeed, the budget in the terminal nodes of the Exploration partition $\N_n^e$ does satisfy this property. %Indeed, assume that the property is true for this given $y$, son of $x$.

%The result is thus satisfied for any $x \in \N_n$.
% 
% Note also that by Lemma~\ref{lem:additiv}, if $\frac{r_x}{T_{x,n}} \leq 4 \frac{\tilde \Sigma}{n}$ for all $x \in \N_n$, then there are enough points for all nodes in $\T_n^e(\N_n)$ to satisfy also this property: the only necessity is that the points are allocated correctly among each child of node $x$.
% 
% By Lemma~\ref{lem:bosup}, we know that for any node $x \in \T_n^e$, we have $\frac{r_x n}{4 \tilde \Sigma} \leq A w_x^{2/3} n^{2/3}$. Note that at the moment of a node's opening, the number of points in the node is smaller than $A w_x^{2/3} n^{2/3}$. As the Exploration stops sampling in a stratum $x$ when $ \frac{r_x}{T_{x,n}+1} \leq  4 \frac{\tilde \Sigma}{n}$, we know that at the end $T$ of the Exploration Phase, for any node $x \in \T^e_n$, we have $\frac{r_x}{T_{x,T}} \geq 4 \frac{\tilde \Sigma}{n}$: it is still possible to allocate the samples correctly among the children after the Exploration Phase.
% 
% Note that after the Exploration Phase, the samples are allocated among the children of $x$ hierarchically, each time to one of the two children nodes such that $\frac{r_y}{T_{y,t}}$ is the bigger.
% By combining the results of the two previous paragraphs, together with the sampling scheme during the exploration phase in any node $x \in \N_n$, we know that

\end{proof}

\begin{lemma}\label{lem:regtree}
 Let $x$ be a node of $\N_n$. Let $\N_x$ be the sub-partition of nodes in $\N_n^e$ that cover the domain of $x$. One has on $\xi$:
\begin{align*}
 \sum_{y \in \N_x} \frac{(w_y\si_y)^2}{T_{y,n}} \leq \frac{(w_x\si_x)^2}{T_{x,n}}.
\end{align*}
\end{lemma}

\begin{proof}
 The result of the Lemma follows by induction.

Let us consider a node $x \in \N_n$, and let $\N_x$ be the sub-partition of nodes in $\N_n^e$ that cover the domain of $x$.

Let $y_1$ and $y_2$ be two nodes of $\N_x$ that have the same father-node $y$. Assume without risk of generality that $r_{y_1} \leq  r_{y_2}$.

Lemma~\ref{lem:egalite} states that
\begin{align*}
 T_{y_1,n} \geq \frac{r_{y_1}}{r_{y_1}+r_{y_2}} (T_{y,n} - 1).
\end{align*}
As $T_{y_1,n} + T_{y_2,n} = T_{y,n}$, we have by the previous Equation
\begin{align*}
 T_{y_2,n} \leq \frac{r_{y_2}}{r_{y_1}+r_{y_2}} (T_{y,n} + 1).
\end{align*}
In the same way, we obtain
\begin{align}\label{eq:y1}
\frac{r_{y_1}}{r_{y_1}+r_{y_2}} (T_{y,n} - 1)\leq  T_{y_1,n} \leq \frac{r_{y_1}}{r_{y_1}+r_{y_2}} (T_{y,n} + 1).
\end{align}
and
\begin{align}\label{eq:y2}
\frac{r_{y_2}}{r_{y_1}+r_{y_2}} (T_{y,n} - 1) \leq  T_{y_2,n} \leq \frac{r_{y_2}}{r_{y_1}+r_{y_2}} (T_{y,n} + 1).
\end{align}
From that we deduce that if $r_{y_1} <  r_{y_2}$, then $T_{y_1,n} \leq T_{y_2,n}$.

If $r_{y_1} = r_{y_2}$, this implies that $|T_{y_2,n} - T_{y_2,n}| \leq 1$, and the last sample is pulled at random between the two strata. From that we deduce that $\frac{(w_{y_1} \si_{y_1})^2}{T_{y_1,n}} + \frac{(w_{y_2} \si_{y_2})^2}{T_{y_2,n}} \leq   \frac{(w_{y} \si_{y})^2}{T_{y,n}}$, in the same way that in Lemma~\ref{lem:uss}.

Assume now that $r_{y_1} < r_{y_2}$. Note now that on $\xi$, because of the definition of $r$, we have on $\xi$
\begin{align*}
 \frac{r_{y_1}}{r_{y_1}+r_{y_2}} \geq \frac{w_{y_1} \si_{y_1}}{w_{y_1} \si_{y_1} + w_{y_2} \si_{y_2}}.
\end{align*}
By combining that with Equation~\ref{eq:y1}, we get on $\xi$
\begin{align*}
 \frac{w_{y_1} \si_{y_1}}{w_{y_1} \si_{y_1} + w_{y_2} \si_{y_2}} (T_{y,n} + 1)\leq  T_{y_1,n},
\end{align*}
which leads to
\begin{align}\label{eq:y11}
 \frac{w_{y_1} \si_{y_1}}{T_{y_1,n}} \leq  \frac{w_{y_1} \si_{y_1} + w_{y_2} \si_{y_2}}{(T_{y,n} + 1)}.
\end{align}

In the same way, as on $\xi$
\begin{align*}
 \frac{r_{y_2}}{r_{y_1}+r_{y_2}} \leq \frac{w_{y_2} \si_{y_2}}{w_{y_1} \si_{y_1} + w_{y_2} \si_{y_2}},
\end{align*}
we have
\begin{align}\label{eq:y22}
 \frac{w_{y_2} \si_{y_2}}{T_{y_2,n}} \geq  \frac{w_{y_1} \si_{y_1} + w_{y_2} \si_{y_2}}{(T_{y,n} - 1)}.
\end{align}
%\todo{huu}

We deduce from Equations~\ref{eq:y11} and~\ref{eq:y22} that on $\xi$
\begin{align*}
 \frac{w_{y_1} \si_{y_1}}{T_{y_1,n}} \leq  \frac{w_{y_2} \si_{y_2}}{T_{y_2,n}}.
\end{align*}
From that, together with the fact that $r_{y_1} <  r_{y_2}$ and $T_{y_1,n} \leq T_{y_2,n}$, we deduce because of variance properties that
\begin{align*}
 \frac{(w_{y_1} \si_{y_1})^2}{T_{y_1,n}}  + \frac{(w_{y_1} \si_{y_2})^2}{T_{y_2,n}} \leq  2\frac{(w_{y_1} \si_{y_1})^2}{T_{y,n}}  + 2\frac{(w_{y_1} \si_{y_2})^2}{T_{y,n}} \leq \frac{(w_{y} \si_{y})^2}{T_{y,n}},
\end{align*}
and note that as $y_1$ and $y_2$ are terminal nodes of $\T^e_n$, then $ \frac{(w_{y_1} \si_{y_1})^2}{T_{y_1,n}}  + \frac{(w_{y_1} \si_{y_2})^2}{T_{y_2,n}}$ correspond to the variance of the stratified estimate on these nodes.

In the same way, by induction, for any child $y$ of $x$ that is in $\T^e_n$, we also have
\begin{align*}
\frac{(w_{y} \si_{y})^2}{T_{y,n}} \geq  \frac{(w_{y_1} \si_{y_1})^2}{T_{y_1,n}}  + \frac{(w_{y_1} \si_{y_2})^2}{T_{y_2,n}} \geq \sum_{z \in \N_x} \frac{(w_x \si_x)^2}{T_{x,n}},
\end{align*}
which is the desired result in the specific case where $y=x$.

\end{proof}

\subsection{Regret of the algorithm}

All the nodes in $\N_n^e$ are sampled in a homogeneous way, so it is coherent to define the risk as
\begin{align*}
 L_n = \sum_{x \in \N_n^e} \frac{(w_x \si_x)^2}{T_{x,n}}.
\end{align*}
By Lemma~\ref{lem:regtree}, we have on $\xi$
\begin{align*}
 L_n = \sum_{x \in \N_n^e} \frac{(w_x \si_x)^2}{T_{x,n}} \leq \sum_{x \in \N_n} \frac{(w_x \si_x)^2}{T_{x,n}}.
\end{align*}
Now by Lemma~\ref{lem:explo}, we have
\begin{align*}
 L_n \leq \sum_{x \in \N_n} \frac{(w_x \si_x)^2}{T_{x,n}} \leq \frac{\Sigma_{\N_n}^2}{n} + B \Sigma_{\N_n} \sum_{y\in{\N_n}}\frac{w_y^{2/3}}{n^{1/3}}.
\end{align*}

Finally, because of Equation~\ref{eq:sigmasigma}
\begin{align*}
 L_n \leq \frac{\Sigma_{\N_n}^2}{n} + B \Sigma_{\N_n} \sum_{y\in{\N_n}}\frac{w_y^{2/3}}{n^{1/3}} \leq \min_{\N} \Bigg[\frac{\Sigma_{\N}^2}{n} + C_{\max}' \Sigma_{\N_n} \sum_{y\in{\N}}\frac{w_y^{2/3}}{n^{1/3}} \Bigg].
\end{align*}
Then by using again that $\N_n$ is the empiric minimizer of the bound, i.e.~Equation~\ref{eq:sigmasigma}, and also by upper bounding $C_{\max}'$, we obtain the final result.

\section{Large deviation inequalities for independent sub-Gaussian random variables}\label{s:tools}

We first state Bernstein inequality for large deviations of independent random variables around their mean.

\begin{lemma}\label{lem:bernstein}
 Let $(X_1,\ldots,X_n)$ be $n$ independent random variables of mean $(\mu_1,\ldots,\mu_n)$ and of variance $(\si_1^2,\ldots,\si_n^2)$. Assume that there exists $b>0$ such that for any $\lambda < \frac{1}{b}$, for any $i \leq n$, it holds that $\E\Big[ \exp(\lambda (X_i-\mu_i)) \Big] \leq \exp\Big( \frac{\lambda^2 \sigma_i^2}{2(1 - \lambda b)}\Big)$. Then with probability $1-\delta$
\begin{equation*}
|\frac{1}{n}\sum_{i=1}^nX_i - \frac{1}{n}\sum_{i=1}^n \mu_i| \leq \sqrt{\frac{2(\frac{1}{n}\sum_{i=1}^n \si_i^2)\log(2/\delta)}{n}} + \frac{b\log(2/\delta)}{n}.
\end{equation*}
\end{lemma}
\begin{proof}
If the assumptions of Lemma~\ref{lem:bernstein} are satisfied, then
\begin{eqnarray*}
 \P \Big( \sum_{i=1}^n X_i - \sum_{i=1}^n \mu_i \geq n \epsilon \Big) &=  \P \Bigg[ \exp\Big(\lambda(\sum_{i=1}^n X_i - \sum_{i=1}^n \mu_i)\Big) \geq \exp(n\lambda \epsilon) \Bigg]\\
&\leq \E\Bigg[\frac{\exp\Big(\lambda(\sum_{i=1}^n X_i - \sum_{i=1}^n \mu_i)\Big)}{\exp(n\lambda \epsilon)}\Bigg]\\
&\leq \prod_{i=1}^n \E\Bigg[\frac{\exp\Big(\lambda(X_i - \mu_i)\Big)}{\exp(\lambda \epsilon)}\Bigg]\\
&\leq \exp(\frac{\lambda^2}{2}\sum_{i=1}^n \frac{\si_i^2}{2(1-\lambda b)} - n\lambda\epsilon).
\end{eqnarray*}

By setting $\lambda = \frac{n\epsilon}{\sum_{i=1}^n\sigma_i^2+bn\epsilon}$ we obtain
\begin{equation*}
 \P \Big(  \sum_{i=1}^n X_i -  \sum_{i=1}^n\mu_i \geq  n\epsilon \Big) \leq \exp(-\frac{n^2\epsilon^2}{2(\sum_{i=1}^n\si_i^2+bn\epsilon)}).
\end{equation*}

By an union bound we obtain
\begin{equation*}
 \P \Big(  |\sum_{i=1}^nX_i - \sum_{i=1}^n \mu_i| \geq  n\epsilon \Big) \leq 2\exp(-\frac{n^2\epsilon^2}{2(\sum_{i=1}^n\si_i^2+bn\epsilon)}).
\end{equation*}

This means that with probability $1-\delta$,
\begin{equation*}
|\frac{1}{n}\sum_{i=1}^nX_i - \frac{1}{n}\sum_{i=1}^n \mu_i| \leq \sqrt{\frac{2(\frac{1}{n}\sum_{i=1}^n \si_i^2)\log(2/\delta)}{n}} + \frac{b\log(2/\delta)}{n}.
\end{equation*}
\end{proof}

We also state the following Lemma on large deviations for the variance of independent random variables.

\begin{lemma}\label{ss:variance}
 Let $(X_1,\ldots,X_n)$ be $n$ independent random variables of mean $(\mu_1,\ldots,\mu_n)$ and of variance $(\si_1^2,\ldots,\si_n^2)$. Assume that there exists $b>0$ such that for any $\lambda < \frac{1}{b}$, for any $i \leq n$, it holds that $\E\Big[ \exp(\lambda (X_i-\mu_i)) \Big] \leq \exp\Big( \frac{\lambda^2 \sigma_i^2}{2(1 - \lambda b)}\Big)$ and also $\E\Big[ \exp(\lambda (X_i-\mu_i)^2 - \lambda \si_i^2) \Big] \leq \exp\Big( \frac{\lambda^2 \si_i^2}{2(1 - \lambda b)}\Big)$.

Let $V=\frac{1}{n}\sum_i (\mu_i - \frac{1}{n}\sum_i \mu_i)^2 + \frac{1}{n} \sum_n \si_i^2$ be the variance of a sample chosen uniformly at random among the $n$ distributions, and $\hat V = \frac{1}{n} \sum_{i=1}^n \big(X_i - \frac{1}{n}\sum_{j=1}^n X_j \big)^2$ the corresponding empirical variance. Then with probability $1-\delta$,
\begin{equation}
|\sqrt{\hat{V}} - \sqrt{V}| \leq  2\sqrt{\frac{(1 + 3b + 4V)\log(2/\delta)}{n}}. 
\end{equation}
\end{lemma}

\begin{proof}
By decomposing the estimate of the empirical variance in bias and variance, we obtain with probability $1-\delta$
\begin{align*}
\hat{V} =& \frac{1}{n}\sum_i (X_i - \frac{1}{n}\sum_j \mu_j)^2 - (\frac{1}{n}\sum_i X_i -\frac{1}{n} \sum_i \mu_i)^2\\
=& \frac{1}{n}\sum_i (X_i -  \mu_i)^2 + 2 \frac{1}{n}\sum_i (X_i -  \mu_i)\frac{1}{n}\sum_i (\mu_i - \frac{1}{n}\sum_j \mu_j)\\
&+  \frac{1}{n}\sum_i (\mu_i - \frac{1}{n}\sum_j \mu_j)^2  - (\frac{1}{n}\sum_i X_i -\frac{1}{n} \sum_i \mu_i)^2\\
=&\frac{1}{n}\sum_i (X_i -  \mu_i)^2 +  \frac{1}{n}\sum_i (\mu_i - \frac{1}{n}\sum_j \mu_j)^2  - (\frac{1}{n}\sum_i X_i -\frac{1}{n} \sum_i \mu_i)^2.
\end{align*}

We then have by the definition of $V$ that with probability $1-\delta$
\begin{equation}\label{eq:varia}
 \hat{V} - V = \frac{1}{n}\sum_{i=1}^n (X_i -  \mu_i)^2 -\frac{1}{n} \sum_{i=1}^n \si_i^2  - (\frac{1}{n}\sum_i X_i -\frac{1}{n} \sum_i \mu_i)^2.
\end{equation}

If the assumptions of Lemma~\ref{ss:variance} are satisfied, we have with probability $1-\delta$
\begin{align*}
 \P \Big( \sum_{i=1}^n (X_i - \mu_i)^2 - \sum_{i=1}^n\si_i^2 \geq n \epsilon \Big) &=  \P \Bigg[ \exp\Big(\lambda(\sum_{i=1}^n |X_i - \mu_i|^2 - \sum_{i=1}^n \si_i^2)\Big) \geq \exp(n\lambda \epsilon) \Bigg]\\
&\leq \E\Bigg[\frac{\exp\Big(\lambda(\sum_{i=1}^n |X_i - \mu_i|^2 - \sum_{i=1}^n\si_i^2)\Big)}{\exp(n\lambda \epsilon)}\Bigg]\\
&\leq \prod_{i=1}^n \E\Bigg[\frac{\exp\Big(\lambda (|X_i - \mu_i|^2 - \si_i^2)\Big)}{\exp(\lambda \epsilon)}\Bigg]\\
&\leq 2\exp(\frac{\lambda^2}{2}\sum_{i=1}^n \frac{\si_i^2}{2(1-\lambda b)} - n\lambda\epsilon).
\end{align*}

If we take $\lambda = \frac{n\epsilon}{\sum_{i=1}^n\sigma_i^2+nb\epsilon}$ we obtain with probability $1-\delta$
\begin{equation}\label{proba:sibgaussvar}
 \P \Big(  \sum_{i=1}^n (X_i -  \mu_i)^2 - \sum_{i=1}^n \si_i^2 \geq  n\epsilon^2 \Big) \leq \exp(-\frac{n^2\epsilon^2}{2(\sum_{i=1}^n\si_i^2+bn\epsilon)}).
\end{equation}

By a union bound we get with probability $1-\delta$ that
\begin{equation*}
 \P \Big(  |\sum_{i=1}^n(X_i - \mu_i)^2 - \sum_{i=1}^n \si_i^2| \geq  n\epsilon \Big) \leq 2\exp(-\frac{n^2\epsilon^2}{2(\sum_{i=1}^n\si_i^2+bn\epsilon)}).
\end{equation*}

This means that with probability $1-\delta$,
\begin{equation}\label{eq:gdev2}
 |\frac{1}{n}\sum_{i=1}^n(X_i - \mu_i)^2 - \frac{1}{n} \sum_{i=1}^n \si_i^2| \leq \sqrt{\frac{2(\frac{1}{n}\sum_{i=1}^n \si_i^2)\log(2/\delta)}{n}} + \frac{b\log(2/\delta)}{n}.
\end{equation}

Finally, by combining Equations \ref{eq:varia} and \ref{eq:gdev2} with Lemma~\ref{lem:bernstein}, we obtain with probability $1-\delta$
\begin{align*}
 |\hat{V} - V| &\leq \frac{4(\frac{1}{n}\sum_{i=1}^n \si_i^2)\log(2/\delta)}{n} + \frac{2b^2\log(2/\delta)^2}{n^2} +\sqrt{\frac{2(\frac{1}{n}\sum_{i=1}^n \si_i^2)\log(2/\delta)}{n}} + \frac{b\log(2/\delta)}{n}\\
&\leq \sqrt{\frac{2(\frac{1}{n}\sum_{i=1}^n \si_i^2)\log(2/\delta)}{n}} + \frac{(3b + 4\frac{1}{n}\sum_{i=1}^n \si_i^2)\log(2/\delta)}{n}\\
&\leq \sqrt{\frac{2V\log(2/\delta)}{n}} + \frac{(3b + 4V)\log(2/\delta)}{n},
\end{align*}
when $n \geq b\log(2/\delta)$ and because $V \geq \frac{1}{n}\sum_{i=1}^n \si_i^2$.

This implies with probability $1-\delta$ that
\begin{align*}
 &V -  \sqrt{\frac{2V\log(2/\delta)}{n}} + \frac{\log(2/\delta)}{2n}  \leq \hat{V} +  \frac{(3b + 4V)\log(2/\delta)}{n} + \frac{\log(2/\delta)}{2n}\\
&\Leftrightarrow \sqrt{V} - \sqrt{\frac{\log(2/\delta)}{2n}}  \leq \sqrt{\hat{V} +  \frac{(1 + 3b + 4V)\log(2/\delta)}{n}} \\
&\Rightarrow \sqrt{V} - \sqrt{\frac{\log(2/\delta)}{2n}}  \leq \sqrt{\hat{V}} +  \sqrt{\frac{(1 + 3b + 4V)\log(2/\delta)}{n}}\\
&\Rightarrow \sqrt{V}  \leq \sqrt{\hat{V}} +  2\sqrt{\frac{(1 + 3b + 4V)\log(2/\delta)}{n}}.
\end{align*}

On the other hand, we have also with probability $1-\delta$
\begin{align*}
 &\hat{V} \leq V+ \sqrt{\frac{2V\log(2/\delta)}{n}} + \frac{(3b + 4V)\log(2/\delta)}{n}\\
&\Rightarrow \sqrt{\hat{V}} \leq \sqrt{V} +  2\sqrt{\frac{(1 + 3b + 4V)\log(2/\delta)}{n}}.
\end{align*}

Finally, we have with probability $1-\delta$
\begin{equation}
|\sqrt{\hat{V}} - \sqrt{V}| \leq  2\sqrt{\frac{(1 + 3b + 4V)\log(2/\delta)}{n}}. 
\end{equation}
\end{proof}]

\end{document}